\documentclass{article}

\usepackage{iclr2027_conference,times}
\usepackage{amsmath,amssymb,amsthm,bm,mathtools}
\usepackage{booktabs}
\usepackage{graphicx}
\usepackage{subcaption}
\usepackage{placeins}
\usepackage{hyperref}
\hypersetup{pdfborder={0 0 1.25}}
\usepackage{url}

\newtheorem{theorem}{Theorem}[section]
\newtheorem{lemma}{Lemma}[section]
\newtheorem{corollary}[theorem]{Corollary}
\newtheorem{proposition}[theorem]{Proposition}
\theoremstyle{definition}

\theoremstyle{remark}
\newtheorem{remark}{Remark}[section]

\newcommand{\R}{\mathbb{R}}
\newcommand{\one}{\bm{1}}
\newcommand{\zero}{\bm{0}}
\newcommand{\Eqn}{Eqn.~}
\newcommand{\norm}[1]{\left\lVert #1\right\rVert}
\newcommand{\supnorm}[1]{\left\lVert #1\right\rVert_{L^\infty(K)}}

\title{\raggedright Activation-Adjacent Layer Normalization: Effects on Approximation Power}
\author{%
\begin{tabular}{c}
\textbf{Yunhao Ni$^{1}$, Yuxin Guo$^{1}$, Jie Luo$^{1}$, Wenjun Wu$^{1,2}$, Lei Huang$^{1,2}$}\thanks{Corresponding author: \texttt{huangleiAI@buaa.edu.cn}.}\\[0.7em]
{\normalfont $^{1}$State Key Laboratory of Complex and Critical Software Environment,}\\
{\normalfont Beihang University, Beijing, China}\\[0.2em]
{\normalfont $^{2}$Hangzhou International Innovation Institute, Beihang University, Hangzhou, China}
\end{tabular}}

\iclrfinalcopy

\begin{document}
\maketitle
\lhead{Preprint}

\begin{abstract}
\hyphenpenalty=5000\exhyphenpenalty=5000
Layer normalization (LN) and root mean square normalization (RMSNorm) depend
on their inputs and remain nonlinear at inference. Inserting either next to
an activation may therefore change the function a network represents. We
study when shallow activation-normalized (AN) and preactivation-normalized
(PN) networks can approximate a prescribed width-$N$ one-hidden-layer MLP.
Viewing RMSNorm as input-dependent scaling, we use width
augmentation to locally linearize normalization on compact feature sets. For
AN networks, an unbounded activation range allows width $N+1$, while every
continuous nonconstant activation admits width $N+r$ with error
$O_f(r^{-1})$. For PN networks, affine parameters or, without them,
degree-one positive homogeneity recover width $N+1$; general continuous
activations admit width $N+r$ with error tending to zero. The corresponding
LN constructions use centered coordinate pairs, yielding widths $N+2$ and
$N+2r$, respectively. We prove that $N+1$ for RMSNorm and $N+2$ for LN are also the
worst-case bound for approximations, and show that applicable universal-approximation theorems
for ordinary MLPs transfer to these normalized architectures.
Numerical experiments recover the predicted approximation rates for both
normalizations across all five settings and illustrate the shifted-ReLU
finite-width obstruction.
\end{abstract}

\section{Introduction}

Classical universal approximation results~\citep{cybenko1989approximation,
hornik1991approximation,leshno1993nonpolynomial} provide a basic foundation for
neural networks: sufficiently expressive classes can approximate large families
of target functions.  Modern architectures augment these standard networks
with components such as residual connections~\citep{he2016deep} and
normalization~\citep{ioffe2015batch,ba2016layer}, often to improve optimization.
The representational effects of residual architectures have also been
studied~\citep{lu2017expressive,lin2018resnet}.  Far less is known about how
normalization changes the function class represented by a network.

Among normalization mechanisms, this representational question is especially
subtle for layer normalization (LN)~\citep{ba2016layer} and its centering-free
variant, root mean square normalization (RMSNorm)~\citep{zhang2019rmsnorm}.
Both recompute their statistics from
each representation and therefore remain nonlinear at inference.  Recent
theory has primarily treated this nonlinearity as a source of expressivity.
\citet{ni2024nonlinearity} describe exact no-affine LN geometrically as
centering followed by scaling. Centering projects the representation onto the
zero-mean hyperplane. Scaling then maps the centered vector radially from the
origin onto a sphere within that hyperplane. The nonlinear step is therefore
RMSNorm applied to the centered vector, not to the original representation.
Exploiting this scaling,
\citet{ni2024nonlinearity} show that width-$3$ LN networks can fit arbitrary labels on finite samples
and derive a depth-dependent VC-dimension lower bound.  Building on this view,
\citet{ni2025pln} arrange multiple LN blocks with linear maps in parallel layer
normalization (PLN), treating LN itself as an approximation primitive to prove
universal approximation.  Whether adding normalization preserves a function
already represented by an MLP with its own coordinate-wise activation has
received little attention.

To study this preservation problem, we insert a single LN or RMSNorm after or
before the activation function of a prescribed one-hidden-layer MLP, obtaining
the activation-normalized (AN) architectures in
\Eqn\ref{eq:rms-an-network} and \Eqn\ref{eq:post-ln-network}, and the
preactivation-normalized (PN) architectures in
\Eqn\ref{eq:rms-pn-network} and \Eqn\ref{eq:pre-ln-network}.  Directly analyzing whether these architectures preserve the approximation
power of the underlying MLP is challenging because LN and RMSNorm couple the
hidden neurons~\citep{ni2024nonlinearity}. Therefore, we adopt a relative-approximation perspective and ask when such a
normalized architecture can approximate a prescribed target MLP.

For RMSNorm, the only distortion is input-dependent scaling. Inspired by the
geometry in \citet{ni2024nonlinearity}, we show that width augmentation moves a
bounded feature set away from the origin, thereby shrinking the angle it
subtends there and making the scaling factor nearly constant after an
input-independent rescaling (Figure~\ref{fig:ln-local-linearization}).
Section~\ref{sec:an-rms} uses this geometric mechanism to construct AN
approximations.

\begin{figure}[ht]
    \centering
    \includegraphics[width=0.94\textwidth]{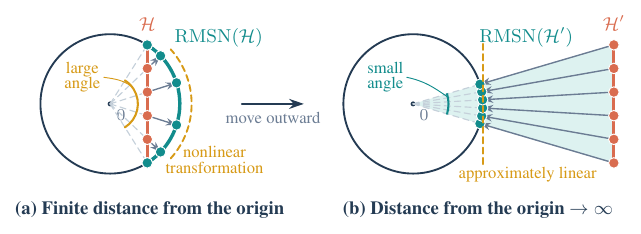}
    \caption{Near- and far-field geometry of RMSNorm. In (a), radial
    normalization maps equally spaced points in $\mathcal H$ to unevenly
    spaced points on the sphere. In (b), the equally sized set $\mathcal H'$
    lies farther from the origin. It subtends a smaller angle, and its
    normalized image occupies a shorter, nearly flat arc.}
    \label{fig:ln-local-linearization}
\end{figure}

For LN, centering introduces a second distortion. We absorb it on the linear
side of the activation: by imposing zero-mean preactivations in PN networks and
an equivalent zero-sum condition on the readout in AN networks. Width
augmentation then makes the remaining input-dependent scaling nearly constant.
Section~\ref{sec:layernorm-analysis} derives the corresponding LN counterparts
of the RMSNorm constructions.

Section~\ref{sec:lower-bounds-universality} establishes that the width-$N+1$
RMSNorm and width-$N+2$ LN bounds are worst-case sharp and shows that applicable
universal-approximation theorems for ordinary MLPs transfer to the normalized
architectures. The experiments in Section~\ref{sec:experiments} recover the
predicted approximation rates.

We provide an overview of our theoretical results in
Table~\ref{tab:result-map}. Here $\checkmark$, $\times$, and
$\checkmark/\times$ indicate that affine normalization parameters are present,
absent, or arbitrary, respectively. The two AN forms coincide, and affine
RMSNorm contains only a coordinate-wise scale. In the lower-bound position,
$\text{--}$ indicates that no matching lower bound is claimed.

\begin{table}[!htb]
\centering
\small
\setlength{\tabcolsep}{5.5pt}
\caption{Width and error guarantees for a width-$N$ target.  Entries are
constructive upper bounds/worst-case lower bounds (UB/LB); matching entries are
sharp, and $\text{--}$ denotes that no matching lower bound is claimed.  All
rows assume continuous $\phi$.}
\label{tab:result-map}
\begin{tabular}{@{}lllccc@{}}
\toprule
Mode & Parameters & Additional condition on $\phi$ & RMSNorm (UB/LB) & LN (UB/LB) & Error\\
\midrule
AN
& $\checkmark/\times$
& unbounded range
& $N+1\,/\,N+1$ & $N+2\,/\,N+2$ & $<\varepsilon$\\
AN
& $\checkmark/\times$
& nonconstant
& $N+r\,/\,\text{--}$ & $N+2r\,/\,\text{--}$ & $\le C_f/r$\\
PN
& $\checkmark$
& none
& $N+1\,/\,N+1$ & $N+2\,/\,N+2$ & $<\varepsilon$\\
PN
& $\times$
& degree-one homogeneous
& $N+1\,/\,N+1$ & $N+2\,/\,N+2$ & $<\varepsilon$\\
PN
& $\times$
& none
& $N+r\,/\,\text{--}$ & $N+2r\,/\,\text{--}$ & $\to0$\\
\bottomrule
\end{tabular}
\end{table}

\section{Related work}

Classical universal approximation theory studies when a neural-network class is
dense in a target function space.  Foundational shallow-network
results~\citep{cybenko1989approximation,hornik1991approximation,
leshno1993nonpolynomial} cover sigmoidal activations and, more generally, broad
classes of nonpolynomial activations.  Work on approximation rates and
minimum-width conditions~\citep{barron1993universal,park2021minimum} addresses
related but distinct questions.  Width-limited fully connected ReLU networks
exhibit a non-universality threshold~\citep{lu2017expressive}, whereas deep
ResNets with one-neuron residual blocks can be universal in
$L^1$~\citep{lin2018resnet}.  Together, these results show that additional
architectural structure can materially change the function class represented
by a network.

Batch normalization~\citep{ioffe2015batch} and
LN~\citep{ba2016layer} were originally introduced to improve neural-network
training.  With fixed population statistics,
inference-time BN is a channelwise affine transformation and can be folded
into a neighboring affine layer.  LN is fundamentally different: its mean and
variance depend on the current representation, so it remains nonlinear at
inference.  \citet{ni2024nonlinearity} decompose LN into centering and scaling.
Geometrically, centering projects onto the zero-mean hyperplane, whereas scaling
projects onto a sphere.  Since centering is linear, their constructions exploit
the nonlinear scaling to show that width-$3$ LN networks can fit arbitrary
labels on finite samples and to derive a depth-dependent VC-dimension lower
bound.  Building on LN as a nonlinear approximation primitive, parallel layer
normalization networks~\citep{ni2025pln} combine multiple LN blocks with linear
maps and establish universal approximation directly for that architecture.

Unlike these studies, which use normalization itself as a source of expressivity,
we study a preservation problem.  The target MLP already contains a fixed
coordinate-wise activation; after inserting a single RMSNorm or LN, we determine
how much additional width suffices to recover its specified function.  Our
constructions use the centering--scaling decomposition in the opposite
direction: width augmentation makes scaling nearly input independent, while
preactivation or readout reparameterization absorbs LN's centering.  The lower
bounds show that this extra width can be unavoidable, and the universal
approximation results follow by transferring classical density through the
relative approximation rather than by constructing targets from LN blocks.

\section{Notation and preliminaries}

In this paper, ordinary italic letters such as $x$ denote scalars, bold lowercase
letters such as $\bm{x}$ denote vectors, and bold uppercase letters such as
$\bm{W}$ denote matrices.  We write $\one_d\in\R^d$ for the all-ones column
vector.  Unless stated otherwise, network inputs and outputs lie in
$\R^{d_x}$ and $\R^{d_y}$.

\subsection{Target one-hidden-layer MLPs}

In this paper, we take a prescribed width-$N$ one-hidden-layer MLP without
normalization as the approximation target, namely
\begin{equation}
    f(\bm{x})=\bm{W}_2\phi(\bm{W}_1\bm{x}+\bm{b}_1)+\bm{b}_2,
    \label{eq:target-mlp}
\end{equation}
where $\phi:\R\to\R$ acts coordinate-wise,
$\bm{W}_1\in\R^{N\times d_x}$, $\bm{b}_1\in\R^N$,
$\bm{W}_2\in\R^{d_y\times N}$, and $\bm{b}_2\in\R^{d_y}$.
We write $\mathcal M_N(\phi)$ for the class of functions represented by all
such networks.  In this paper, we take $\mathcal M_N(\phi)$ to be the target
function class.  Approximation is measured on a compact set
$K\subset\R^{d_x}$ using
$\norm{g}_{\infty,K}=\sup_{\bm{x}\in K}\norm{g(\bm{x})}$.  For each
$f\in\mathcal M_N(\phi)$, we fix one width-$N$ realization as in
\Eqn\ref{eq:target-mlp}; constants indexed by $f$ may depend on this realization
and on $K$.

Accordingly, we compare a normalized network directly with a target
$f\in\mathcal M_N(\phi)$.  For $\varepsilon>0$, a width-$d$ normalized
network $\widehat f$ is an $\varepsilon$-approximation of $f$ on $K$ if
\begin{equation}
    \norm{f-\widehat f}_{\infty,K}<\varepsilon.
    \label{eq:relative-approximation}
\end{equation}

\subsection{RMS normalization}
\label{subsec:rmsn-definition}

We begin with normalization by scale alone.  For $\bm{h}\in\R^d$, RMSN is
defined as $\sqrt d\,\bm{h}/\norm{\bm{h}}$ when $\bm{h}\ne\zero$ and as
$\zero$ when $\bm{h}=\zero$.
The map removes dependence on the input magnitude and sends every nonzero input
to the sphere $\{\bm{u}\in\R^d:\norm{\bm{u}}^2=d\}$. For
$\diamond\in\{0,\mathrm{aff}\}$, let
$\operatorname{RMSN}^{(0)}(\bm{h})=\operatorname{RMSN}(\bm{h})$ and
$\operatorname{RMSN}^{(\mathrm{aff})}(\bm{h})
=\bm{\alpha}\odot\operatorname{RMSN}(\bm{h})$, where
$\bm{\alpha}\in\R^d$. The affine mode has a learnable coordinate-wise scale
but no additive bias.

The next theorem shows that one extra coordinate lets RMSN approximate any
linear map on a compact feature set.

\begin{theorem}[Approximation of linear maps by RMSN]
\label{thm:lifted-rmsn-linearization}
Let $\mathcal H\subset\R^N$ be compact and let $\bm{W}^\star\in\R^{d_y\times N}$.  For every
$M>0$, there exists a width-$(N+1)$ map
 \begin{equation}
 f_M(\bm{h})
 =\bm{W}_2\operatorname{RMSN}
 (\bm{W}_1\bm{h}+\bm{b}_1)+\bm{b}_2
 \label{eq:rmsn-linear-network}
\end{equation}
such that
\begin{equation}
    \sup_{\bm{h}\in\mathcal H}
    \norm{f_M(\bm{h})-\bm{W}^\star\bm{h}}
    =O(M^{-2}).
    \label{eq:lifted-rmsn-linearization-rate}
\end{equation}
\end{theorem}

\noindent
\begin{minipage}[t]{0.65\textwidth}
    \vspace{0pt}
    The parameters may depend on $M$.  They can be chosen with $\bm{W}_1$
    independent of $M$, $\norm{\bm{b}_1}_\infty=M$, and
    $\norm{\bm{W}_2}_{\max}=O(M)$, where
    $\norm{\bm{A}}_{\max}:=\max_{i,j}|A_{ij}|$.

    The exact identity and proof are given in
    Appendix~\ref{app:lifted-rmsn-proof}.  The construction augments the width by
    one, lifting each $\bm{h}\in\mathcal H$ to $[\bm{h}^\top,M]^\top$.  Because $\mathcal H$ is
    bounded, the augmented family moves away from the origin and the angle it
    subtends there tends to zero as $M\to\infty$.
    Figure~\ref{fig:ln-local-linearization} illustrates this far-field geometry.
    After an input-independent rescaling of
    order $M$, the resulting RMSN map approaches a linear map uniformly on $\mathcal H$.
    The appended coordinate receives zero readout weight, so the decoder recovers
    $\bm{W}^\star\bm{h}$.  Figure~\ref{fig:lift-linearize-decode} summarizes the
    construction.  This lift-linearize-decode mechanism provides the basic template
    for the constructions that follow, with additional reparameterizations used to
    accommodate the activation and, in the LN case, centering.
\end{minipage}\hfill
\begin{minipage}[t]{0.32\textwidth}
    \vspace{0.5pt}
    \centering
    \captionsetup{font=footnotesize,skip=2pt,hypcap=false}
    \includegraphics[width=\linewidth]{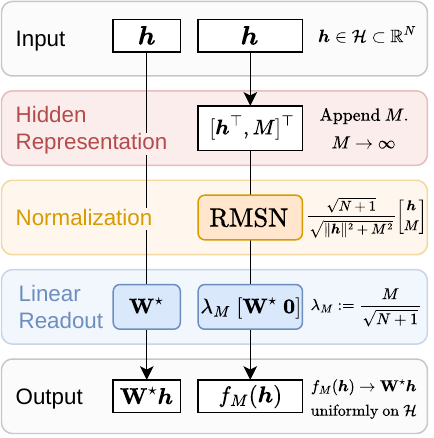}
    \captionof{figure}{Target map and its one-coordinate RMSNorm realization
    (Theorem~\ref{thm:lifted-rmsn-linearization}).}
    \label{fig:lift-linearize-decode}
\end{minipage}
\par\medskip

We next study whether inserting RMSNorm immediately before or after the
activation of a target one-hidden-layer MLP preserves its represented
function.

\section{Activation-normalized RMS networks}
\label{sec:an-rms}

The width-$N$ activation-normalized RMS class is
\begin{equation}
 \mathcal R_N^{\mathrm{AN},\diamond}(\phi)
 :=\left\{\bm{x}\mapsto
 \bm{W}_2\operatorname{RMSN}^{(\diamond)}
 \!\left(\phi(\bm{W}_1\bm{x}+\bm{b}_1)\right)+\bm{b}_2\right\}.
 \label{eq:rms-an-network}
\end{equation}
Here and below all displayed weights, biases, and available normalization
parameters vary freely.  The affine scale is absorbed into $\bm{W}_2$, so the
affine and no-affine AN classes coincide; we denote their common value by
$\mathcal R_N^{\mathrm{AN}}(\phi)$.

The common proof strategy is to augment the width until the original feature
map is recovered, up to a uniformly vanishing scaling distortion, from the first
$N$ normalized coordinates.  The added coordinates control the normalization
statistics and are ignored by the target readout, as illustrated in
Figure~\ref{fig:lift-linearize-decode}.

\subsection{Unbounded activations: magnitude controls the normalization scale}

\begin{theorem}[Approximation by RMS-AN networks with unbounded activations]
\label{thm:rms-an-unbounded}
Let $\phi$ be continuous with unbounded range.  For every
$f\in\mathcal M_N(\phi)$, there are constants $C_f,R_f>0$ such that every
attainable value $M\in\phi(\R)$ with $|M|\ge R_f$ admits
\begin{equation}
 \widehat f_M\in\mathcal R_{N+1}^{\mathrm{AN}}(\phi),
 \qquad
 \supnorm{\widehat f_M-f}\le \frac{C_f}{M^2}.
 \label{eq:rms-an-unbounded-rate}
\end{equation}
\end{theorem}

The construction appends the activated value $M$ and assigns its coordinate
zero readout weight.  After a fixed rescaling, the signal is multiplied by
$(1+\norm{\bm{z}(\bm{x})}^2/M^2)^{-1/2}$, which converges uniformly to one.
This is exactly the mechanism in Figure~\ref{fig:lift-linearize-decode}.

Driving the error in \Eqn\ref{eq:rms-an-unbounded-rate} to zero requires
attainable values $|M|\to\infty$, which is why this one-coordinate construction
needs an activation with unbounded range.  For bounded activations this regime
is unavailable; the next subsection replaces magnitude by multiplicity.

\subsection{General activations: multiplicity replaces magnitude}

\begin{theorem}[Approximation by RMS-AN networks with general activations]
\label{thm:rms-an-general}
Let $\phi$ be continuous and nonconstant.  For every
$f\in\mathcal M_N(\phi)$, there are constants $C_f>0$, $r_f\in\mathbb N$, and
a sequence $\widehat f_r\in\mathcal R_{N+r}^{\mathrm{AN}}(\phi)$ such that
\begin{equation}
 \supnorm{\widehat f_r-f}\le\frac{C_f}{r}
 \qquad(r\ge r_f).
 \label{eq:rms-an-general-rate}
\end{equation}
\end{theorem}

Choose any nonzero attainable value $a$ and append $r$ copies.  Their squared
norm grows like $ra^2$, while the target feature vector remains bounded, so the
remaining scaling factor is $(1+\norm{\bm{z}(\bm{x})}^2/(ra^2))^{-1/2}$.
Appendix~\ref{app:rms-an-proofs} proves both theorems.

\paragraph{Summary.}
RMS-AN networks require only scale stabilization.  One large attainable value
gives width $N+1$, while $r$ repeated bounded values give width $N+r$.  The
next section moves normalization before the activation, where its scale can no
longer generally be absorbed into the readout.

\section{Preactivation-normalized RMS networks}
\label{sec:pn-rms}

Unlike in AN networks, the affine RMSNorm scale in PN networks acts before the
nonlinear activation and therefore cannot generally be absorbed into the
output linear map.  The affine and no-affine PN classes can therefore differ,
so we treat them separately.  The corresponding width-$N$ PN classes are
\begin{equation}
 \mathcal R_N^{\mathrm{PN},\diamond}(\phi)
 :=\left\{\bm{x}\mapsto
 \bm{W}_2\phi\!\left(\operatorname{RMSN}^{(\diamond)}
 (\bm{W}_1\bm{x}+\bm{b}_1)\right)+\bm{b}_2\right\}.
 \label{eq:rms-pn-network}
\end{equation}
\subsection{PN networks with affine parameters}

\begin{theorem}[Approximation by affine PN networks]
\label{thm:rms-affine-pre}
Let $\phi$ be continuous.  For every $f\in\mathcal M_N(\phi)$ and every
$\varepsilon>0$, there exists
$\widehat f\in\mathcal R_{N+1}^{\mathrm{PN,aff}}(\phi)$ such that
$\supnorm{\widehat f-f}<\varepsilon$.  If $\phi$ is Lipschitz on compact sets,
the approximants have error $O_f(M^{-2})$ as $M\to\infty$.
\end{theorem}

The proof follows the construction of
Theorem~\ref{thm:lifted-rmsn-linearization}, with the affine RMSNorm scale
performing the required linear rescaling before $\phi$.  For each target
preactivation $\bm{t}$, append the constant $M$ as one extra coordinate and
choose the affine scale so that the first $N$ normalized coordinates are
$\left(1+\norm{\bm{t}}^2/M^2\right)^{-1/2}\bm{t}$.  Because $\bm{t}$ ranges
over a compact set, the scalar factor converges uniformly to one as
$|M|\to\infty$.  The continuity of $\phi$ therefore transfers this identity
limit through the activation.
Appendix~\ref{app:rms-pn-proofs} gives the proof.

\subsection{A finite-width obstruction without affine parameters}

\begin{proposition}[A finite-width obstruction]
\label{prop:fixed-width-obstruction}
For every finite normalized width $d$, let
$\psi_d(t)=\max\{0,t-\sqrt d\}$.  Every width-$d$ no-affine RMS-PN network with
activation $\psi_d$ is constant, whereas ordinary $\psi_d$-networks can
represent nonconstant functions.
\end{proposition}

Every coordinate of an RMSNorm output lies in $[-\sqrt d,\sqrt d]$, where
$\psi_d$ vanishes.  Thus no fixed $d$ gives a guarantee uniform over all
continuous activations: at finite width, normalization can erase the
activation's nonconstant regime and hence its expressive power.  The next two
subsections give two distinct ways around this obstruction.  Positive
homogeneity makes an inverse rescaling of the second-layer weights equivalent
to applying the required scale before $\phi$, retaining width $N+1$;
for a general continuous activation, increasing width recovers approximation
asymptotically.

\subsection{No-affine PN networks with positively homogeneous activations}

\begin{theorem}[Approximation by no-affine PN networks with homogeneous activations]
\label{thm:rms-homogeneous-pre}
Suppose that $\phi$ is continuous and positively homogeneous of degree one,
namely, $\phi(\lambda t)=\lambda\phi(t)$ for all $\lambda\ge0$ and
$t\in\R$.
For every $f\in\mathcal M_N(\phi)$ and every $\varepsilon>0$, there exists
$\widehat f\in\mathcal R_{N+1}^{\mathrm{PN,0}}(\phi)$ such that
$\supnorm{\widehat f-f}<\varepsilon$.  The approximants have error
$O_f(M^{-2})$ as $M\to\infty$.
\end{theorem}

Without affine parameters, normalization leaves a positive common factor before
$\phi$, so no homogeneity condition for negative multipliers is needed.  Positive
homogeneity moves this factor through the activation, and an inverse rescaling
of the second-layer weights cancels it.  Equivalently, the output-layer
rescaling supplies the scale that affine RMSNorm would have applied before $\phi$.

\subsection{No-affine PN networks with general continuous activations}

\begin{theorem}[Approximation by no-affine PN networks with general activations]
\label{thm:rms-plain-pre-general}
Let $\phi$ be continuous.  For every $f\in\mathcal M_N(\phi)$, there is a
sequence $\widehat f_r\in\mathcal R_{N+r}^{\mathrm{PN,0}}(\phi)$ such that
$\supnorm{\widehat f_r-f}\to0$.  If $\phi$ is Lipschitz on compact sets, then
for some $C_f>0$ and $r_f\in\mathbb N$,
\begin{equation}
 \supnorm{\widehat f_r-f}\le\frac{C_f}{r}
 \qquad(r\ge r_f).
 \label{eq:rms-plain-pre-rate}
\end{equation}
\end{theorem}

For the target preactivation $\bm{t}(\bm{x})$, append $r$ constant coordinates
equal to one, forming
$[\bm{t}(\bm{x})^\top,\one_r^\top]^\top\in\R^{N+r}$.  By the definition of
RMSNorm, its first $N$ normalized coordinates are
$\left((N+r)/(\norm{\bm{t}(\bm{x})}^2+r)\right)^{1/2}\bm{t}(\bm{x})$.  On the
compact input set, the scalar factor converges uniformly to one, so the
normalized preactivation converges to the original one.  The networks exist
for every $r$, but the displayed error bound is asserted only for $r\ge r_f$.
Appendix~\ref{app:rms-pn-proofs} gives the proofs of all results in this
section.

\paragraph{Summary.}
Affine parameters give an $N+1$ RMS-PN approximation for every continuous
activation, and positive homogeneity retains width $N+1$ without affine
parameters.  General
continuous activations instead require replicated-coordinate width growth.
The next section adds centering and derives the corresponding LayerNorm
guarantees.

\section{Adding centering: From RMSNorm to LayerNorm}
\label{sec:layernorm-analysis}

LayerNorm~\citep{ba2016layer} differs from RMSNorm by an additional centering
operation: before normalization by scale, it subtracts the coordinate mean.  Let
$\bm{P}_d=\bm{I}_d-d^{-1}\one_d\one_d^\top$.  With this notation, LN is
\begin{equation}
 \operatorname{LN}(\bm{h})
 =\operatorname{RMSN}(\bm{P}_d\bm{h})
 =\frac{\bm{h}-\mu(\bm{h})\one_d}{\sigma(\bm{h})},
 \label{eq:exact-ln}
\end{equation}
where $\mu(\bm{h})=d^{-1}\one_d^\top\bm{h}$ and
$\sigma(\bm{h})=(d^{-1}\norm{\bm{h}-\mu(\bm{h})\one_d}^2)^{1/2}$.
We set $\operatorname{LN}(\bm{h})=\zero$ when $\bm{h}$ has zero variance. Its
affine mode is
$\operatorname{LN}_{\bm{\alpha},\bm{\beta}}(\bm{h})
=\bm{\alpha}\odot\operatorname{LN}(\bm{h})+\bm{\beta}$ for
$\bm{\alpha},\bm{\beta}\in\R^d$. For $\diamond\in\{0,\mathrm{aff}\}$, let
$\operatorname{LN}^{(0)}=\operatorname{LN}$ and
$\operatorname{LN}^{(\mathrm{aff})}
=\operatorname{LN}_{\bm{\alpha},\bm{\beta}}$.

For $\diamond\in\{0,\mathrm{aff}\}$, define the LN classes
\begin{align}
 \mathcal P_N^{\mathrm{AN},\diamond}(\phi)
 &:={}
 \left\{\bm{x}\mapsto
 \bm{W}_2\operatorname{LN}^{(\diamond)}
 \!\left(\phi(\bm{W}_1\bm{x}+\bm{b}_1)\right)+\bm{b}_2\right\},
 \label{eq:post-ln-network}\\
 \mathcal P_N^{\mathrm{PN},\diamond}(\phi)
 &:={}
 \left\{\bm{x}\mapsto
 \bm{W}_2\phi\!\left(\operatorname{LN}^{(\diamond)}
 (\bm{W}_1\bm{x}+\bm{b}_1)\right)+\bm{b}_2\right\},
 \label{eq:pre-ln-network}
\end{align}
where $\operatorname{LN}^{(0)}=\operatorname{LN}$ and
$\operatorname{LN}^{(\mathrm{aff})}=\operatorname{LN}_{\bm{\alpha},\bm{\beta}}$.
For AN networks the affine parameters are absorbed into the readout, so
$\mathcal P_N^{\mathrm{AN,aff}}(\phi)=\mathcal P_N^{\mathrm{AN,0}}(\phi)$;
we denote the common class by $\mathcal P_N^{\mathrm{AN}}(\phi)$.

The width-augmentation principle is the same as for RMSNorm: the added
coordinates make the normalization scale nearly constant on the bounded signal
family.  The new issue is LN's centering projection $\bm{P}_d$, which we absorb
through reparameterization rather than by approximating the subtracted mean.
Theorem~\ref{thm:lifted-rmsn-linearization} extends to LN by replacing its
one-coordinate augmentation with the following mean-zero two-coordinate
construction.  Write
$s(\bm{h})=\one_N^\top\bm{h}$ and replace the one-coordinate augmentation by
$\bm{y}_M(\bm{h})=[\bm{h}^\top,-s(\bm{h})/2+M,-s(\bm{h})/2-M]^\top$.
Since $\one_{N+2}^\top\bm{y}_M(\bm{h})=0$, LN equals RMSNorm on the augmented
vector; the same decoder gives width $N+2$ and error $O(M^{-2})$.
Appendix~\ref{app:lifted-ln-proof} gives the exact calculation.

\subsection{Centering through reparameterization}

The centering operation can be absorbed on the side of the activation where
the relevant map remains linear.  For PN networks, the mean-zero
two-coordinate construction above absorbs centering before the activation and
retains the $O(M^{-2})$ rate; see Appendix~\ref{app:affine-pre-proof}.

For AN networks, this input-side argument is unavailable because generally
$\phi(\bm{P}_d\bm{h})\ne\bm{P}_d\phi(\bm{h})$.  Moreover, an activation such as
ReLU may be unbounded only on one side, so the mean-zero pair $M,-M$ used above
need not be attainable.  The following result shows how one attainable large
value still suffices.

\begin{theorem}[Approximation by LN-AN networks with unbounded activations]
\label{thm:ln-an-unbounded}
Let $\phi$ be continuous with unbounded range.  For every
$f\in\mathcal M_N(\phi)$, there are constants $C_f,R_f>0$ such that every
attainable value $M\in\phi(\R)$ with $|M|\ge R_f$ admits
\begin{equation}
 \widehat f_M\in\mathcal P_{N+2}^{\mathrm{AN}}(\phi),
 \qquad
 \supnorm{\widehat f_M-f}\le \frac{C_f}{|M|}.
 \label{eq:ln-an-unbounded-rate}
\end{equation}
\end{theorem}

\paragraph{Proof analysis.}
Let the target be the width-$n$ MLP
$g(\bm{x})=\bm{A}\phi(\bm{B}\bm{x}+\bm{b})+\bm{a}
=\bm{A}\bm{z}(\bm{x})+\bm{a}$.  We first append a single attainable value $M$
with zero readout weight.  For clarity take $M>0$; the negative case has the
same calculation with the corresponding signs.  Put $d=n+1$,
$s=\one_n^\top\bm{z}$, $\mu_M=(s+M)/d$,
$\sigma_M=\sigma([\bm{z}^\top,M]^\top)$,
$\tau_M=M\sqrt n/d$, and $\lambda_M=\tau_M/\sigma_M$.  The candidate network is
\begin{equation}
 \widehat g_M(\bm{x})
 =\tau_M\begin{bmatrix}\bm{A}&\zero\end{bmatrix}
   \operatorname{LN}\!\left(
   \begin{bmatrix}\bm{z}(\bm{x})\\M\end{bmatrix}\right)
   +\bm{a}+\frac{M}{d}\bm{A}\one_n.
 \label{eq:ln-an-one-coordinate-candidate}
\end{equation}
A direct calculation gives the approximation error
\begin{equation}
 \widehat g_M-g
 =(\lambda_M-1)\bm{A}\bm{z}
   -\left(\lambda_M\mu_M-\frac{M}{d}\right)\bm{A}\one_n.
 \label{eq:ln-an-one-coordinate-error}
\end{equation}
The last term in \Eqn\ref{eq:ln-an-one-coordinate-candidate} cancels the leading
input-independent centering error.  Uniformly on the compact input set,
$\lambda_M=1+s/(nM)+O(M^{-2})$ and
$\lambda_M\mu_M-M/d=s/n+O(M^{-1})$.  Thus the remaining error generally tends
to $-(s(\bm{x})/n)\bm{A}\one_n$.  The one-coordinate construction succeeds
precisely under the zero-sum condition $\bm{A}\one_n=\zero$.

\begin{lemma}[Conditional one-coordinate approximation]
\label{lem:ln-an-zero-sum-conditional}
Let $\phi$ be continuous with unbounded range, and let
$g(\bm{x})=\bm{A}\phi(\bm{B}\bm{x}+\bm{b})+\bm{a}$ be a width-$n$ MLP with
$\bm{A}\one_n=\zero$.  There are constants $C_g,R_g>0$ such that every
attainable $M\in\phi(\R)$ with $|M|\ge R_g$ admits
$\widehat g_M\in\mathcal P_{n+1}^{\mathrm{AN}}(\phi)$ with
$\sup_{\bm{x}\in K}\norm{\widehat g_M(\bm{x})-g(\bm{x})}\le C_g/|M|$.
\end{lemma}

The complete proof is given in Appendix~\ref{app:post-unbounded-proof}.
Lemma~\ref{lem:ln-an-zero-sum-conditional} thus approximates a zero-sum target
after appending only one coordinate.  It remains to put an arbitrary target MLP
in this favorable form; one additional constant activated coordinate suffices.

\begin{lemma}[Zero-sum readout reparameterization]
\label{lem:post-zero-sum}
Every $f\in\mathcal M_N(\phi)$ has an exactly equivalent width-$(N+1)$
one-hidden-layer MLP
$\bar f(\bm{x})=\bar{\bm{W}}_2
\phi(\bar{\bm{W}}_1\bm{x}+\bar{\bm{b}}_1)+\bar{\bm{b}}_2$
with $\bar{\bm{W}}_2\one_{N+1}=\zero$.
\end{lemma}

The proof is given in Appendix~\ref{app:feature-centered-readout-proof}.
Combining Lemmas~\ref{lem:ln-an-zero-sum-conditional}
and~\ref{lem:post-zero-sum}, first reparameterizing the target at width $N+1$
and then appending $M$, gives width $N+2$ and proves
Theorem~\ref{thm:ln-an-unbounded}.

\begin{samepage}
\begin{remark}
Viewed in isolation, Lemma~\ref{lem:post-zero-sum} may look inapplicable: its
zero-sum readout acts after scaling, whereas LN centers before scaling.  In our
construction, the large coordinate $M$ makes the RMSNorm-type scale nearly
input independent and hence approximately linear on the lifted compact family.
Because the residual scaling is a common scalar, the readout centering passes
through it and aligns with LN's inner centering.  Thus
Lemma~\ref{lem:post-zero-sum} removes centering exactly, while
Lemma~\ref{lem:ln-an-zero-sum-conditional} bounds the residual scaling error.
\end{remark}
\end{samepage}

\subsection{Remaining LayerNorm approximation guarantees}

\begin{proposition}[Remaining LayerNorm counterparts]
\label{prop:ln-counterparts}
The affine-PN and homogeneous-PN conclusions of
Theorems~\ref{thm:rms-affine-pre} and~\ref{thm:rms-homogeneous-pre} remain valid
for the corresponding LN classes with width $N+2$.  The general-activation
conclusions of Theorems~\ref{thm:rms-an-general}
and~\ref{thm:rms-plain-pre-general} both remain valid with width $N+2r$.
All stated rates are unchanged.  Moreover, the finite-width obstruction in
Proposition~\ref{prop:fixed-width-obstruction} remains valid when RMSNorm is
replaced by LN.
\end{proposition}

The contrast with RMSNorm is explicit in the widths that grow with $r$.  Its
general AN construction repeats one attainable value $r$ times and has width
$N+r$.  After LN centering, copies of a single value cannot provide a
nondegenerate limiting variance, so we repeat each of two distinct attainable
values $r$ times, giving width $N+2r$.  The constant part of the induced mean
is absorbed into the output bias; the remaining input-dependent centering and
scaling errors are both $O(r^{-1})$.  Thus, unlike the one-sided unbounded
construction, this argument does not require
Lemma~\ref{lem:post-zero-sum}.  For PN networks, the analogous replacement of
$r$ identical coordinates by $r$ symmetric pairs also gives width $N+2r$.
The two-coordinate linearization in
Appendix~\ref{app:lifted-ln-proof} supplies the affine and homogeneous
constructions, while replicated symmetric
coordinates handle a general continuous activation.  Appendix~\ref{app:ln-counterparts-proofs}
gives the detailed arguments.

\paragraph{Summary.}
This section extends the five RMSNorm approximation results to LayerNorm by
reparameterizing its centering.  The next section establishes their joint
worst-case width lower bounds and transfers them to universal approximation.

\section{Width lower bounds and universal approximation}
\label{sec:lower-bounds-universality}

\subsection{Joint worst-case width lower bounds}
\label{subsec:joint-lower-bounds}

The $N+1$ RMSNorm and $N+2$ LayerNorm widths are jointly sharp.  In
each case the identity target on the unit ball exposes the obstruction: when
the normalized layer is too narrow, it collapses a nontrivial input segment.

\begin{proposition}[Worst-case necessity of one additional coordinate]
\label{prop:rms-one-coordinate-necessity}
Let $N\ge1$ and $K=\{\bm{x}\in\R^N:\norm{\bm{x}}\le1\}$.  For every $d\le N$
and every width-$d$ map
\begin{equation}
 g(\bm{x})=\bm{W}_2\operatorname{RMSN}
 (\bm{W}_1\bm{x}+\bm{b}_1)+\bm{b}_2\in\R^N,
 \label{eq:rms-lower-bound-map}
\end{equation}
we have $\sup_{\bm{x}\in K}\norm{g(\bm{x})-\bm{x}}\ge\frac12$.
\end{proposition}

The proof uses rank deficiency when $d<N$ and the ray invariance of RMSNorm
when $d=N$.  The same argument applies to LN after replacing the affine map by
its centered version $\bm{P}_d(\bm{W}_1\bm{x}+\bm{b}_1)$, where
$\bm{P}_d=\bm{I}_d-d^{-1}\one_d\one_d^\top$ is the centering projection.  Since
$\operatorname{rank}(\bm{P}_d\bm{W}_1)\le d-1$, every width $d\le N+1$ LN map
likewise has error at least $1/2$, so the first admissible width is $N+2$.

With $\phi(t)=t$, Proposition~\ref{prop:rms-one-coordinate-necessity} and its
LN counterpart show that the $N+1$ RMSNorm and $N+2$ LayerNorm
width bounds are worst-case sharp.  The $N+r$ and $N+2r$ widths remain
constructive upper bounds.  Appendix~\ref{app:rms-one-coordinate-lower-bound}
gives the shared argument.

\subsection{Universal approximation transfer}

The preceding relative-approximation guarantees transfer any applicable
universal-approximation theorem for one-hidden-layer MLPs to the corresponding
normalized classes.

\begin{corollary}[Universal approximation transfer]
\label{cor:density-transfer}
Suppose that $\bigcup_{N\ge1}\mathcal M_N(\phi)$ is dense in
$C(K;\R^{d_y})$ under the uniform norm.  If $\phi$ is continuous and
nonconstant, then the unions of both RMS-AN and LN-AN classes are dense.  If
$\phi$ is continuous, then the corresponding affine and no-affine
PN unions are dense as well.
\end{corollary}

Indeed, approximate the target first by a one-hidden-layer MLP and then
approximate that MLP by the corresponding normalized network, the triangle
inequality gives the claim.  Appendix~\ref{app:density-transfer-proof} gives
the details.

\section{Numerical illustration}
\label{sec:experiments}
\label{sec:numerical-illustration}

We test the three fixed-overhead rate predictions using six independently
sampled width-$16$ target MLPs in each setting.  For each target and
augmentation magnitude $M$, we instantiate the LN and RMSNorm approximants
using the explicit constructions from the proofs, without fitting their
parameters.  We evaluate the maximum absolute output error on a fixed test
grid and plot its mean over the six targets in
Figure~\ref{fig:five-construction-rates}.  In
Figure~\ref{fig:five-construction-rates}(a), the RMSNorm AN error eventually
decreases in line with the $M^{-2}$ rate in
Theorem~\ref{thm:rms-an-unbounded}, whereas the LN AN error follows the
$M^{-1}$ rate in Theorem~\ref{thm:ln-an-unbounded}.  The LN curve drops
sharply between $M=4$ and $M=8$ and then changes little at $M=16$.
The $M^{-1}$ rate is asymptotic and does not prescribe the shape of this
early part of the curve.  Figures~\ref{fig:five-construction-rates}(b)
and~\ref{fig:five-construction-rates}(c) show nearly overlapping LN and RMSNorm
errors consistent with the $M^{-2}$ rates for affine and homogeneous PN in
Theorems~\ref{thm:rms-affine-pre} and~\ref{thm:rms-homogeneous-pre} and their
LN counterparts in Proposition~\ref{prop:ln-counterparts}.  The gray lines
indicate reference slopes, not numerical error bounds.

We also test the remaining width-growing constructions, other activations, and
target widths $N\in\{4,8,16,32\}$.  These experiments exhibit the same predicted
behavior.  Appendix~\ref{app:numerical-details} gives the complete results,
including interquartile ranges and the shifted-ReLU width obstruction.

\FloatBarrier
\begin{figure}[ht]
    \centering
    \begin{subfigure}[t]{0.32\textwidth}
        \centering
        \includegraphics[width=\textwidth]{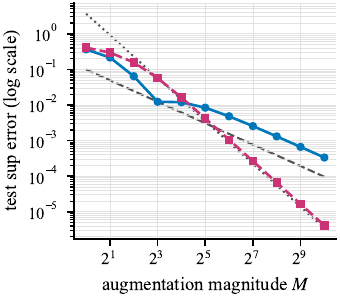}
        \caption{Unbounded AN (ReLU).}
        \label{fig:five-construction-rates-a}
    \end{subfigure}\hfill
    \begin{subfigure}[t]{0.32\textwidth}
        \centering
        \includegraphics[width=\textwidth]{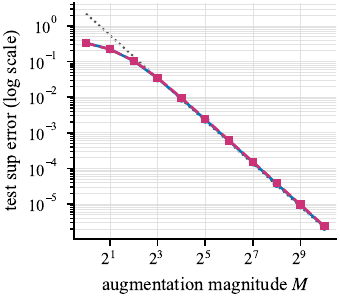}
        \caption{Affine PN ($\tanh$).}
        \label{fig:five-construction-rates-b}
    \end{subfigure}\hfill
    \begin{subfigure}[t]{0.32\textwidth}
        \centering
        \includegraphics[width=\textwidth]{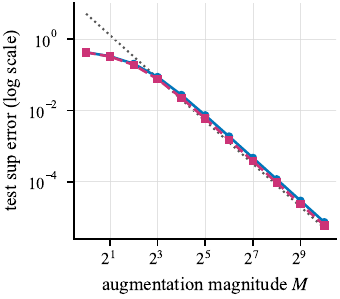}
        \caption{Homogeneous PN (ReLU).}
        \label{fig:five-construction-rates-c}
    \end{subfigure}

    \includegraphics[width=0.58\textwidth]{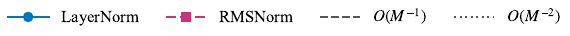}
    \caption{Representative verification of the three $M$-controlled
    constructions at $N=16$.  Each curve reports the mean of grid-based maximum
    absolute errors over
    six target MLPs generated from independent random seeds.  ReLU is used for
    unbounded AN and homogeneous PN, whereas $\tanh$ is used for affine PN.
    Both axes use logarithmic scales.  Markers distinguish LayerNorm and
    RMSNorm, and gray lines illustrate reference slopes rather than error
    bounds.  The complete
    four-width activation sweep and the
    width-growing experiments appear in
    Appendix~\ref{app:numerical-details}.}
    \label{fig:five-construction-rates}
\end{figure}
\FloatBarrier

\section{Conclusion}

Starting from the nonlinearity of input-dependent scaling, we show that width
augmentation can make normalization approximately linear on compact feature
sets.  Under the corresponding assumptions, inserting RMSNorm after or before
the activation of a one-hidden-layer MLP can therefore preserve its represented
function arbitrarily well by increasing width.  LayerNorm admits analogous
guarantees after its centering operation is absorbed by reparameterization.
Applicable universality results transfer, the $N+1$ RMSNorm and $N+2$
LayerNorm bounds are worst-case sharp, and experiments recover the predicted
approximation rates.

\paragraph{Limitations.}
The results concern LN and RMSNorm on compact sets in one shallow
normalization--activation composition. Fixed numerical stabilizers, coupled
normalization layers, residual architectures, and finite-precision effects lie
outside the analysis.  Whether the constructive widths $N+r$ and $N+2r$ are
necessary for standard fixed activations remains open.  The scalar synthetic
experiments do not establish behavior in high-dimensional or deep networks.

\subsection*{Acknowledgments}
This work was supported by the National Natural Science Foundation of China
(62476016, 62441617), and the Fundamental Research Funds for the Central
Universities and the Beijing Natural Science Foundation (Grant No. QY24137).

\subsection*{Code availability}
The experiment code and numerical data are available at
\url{https://github.com/Musicath/AALN-experiments}.

\subsection*{AI use statement}
We used generative AI tools to assist with manuscript editing, proof checking,
literature search and citation checking, code implementation and debugging,
and the preparation and interpretation of synthetic experiments and figures.
The research question, mathematical claims and constructions, and experimental
objectives were formulated by the authors. All AI-assisted outputs were
reviewed and verified by the authors, who take full responsibility for the
final content.

\nocite{*}
\bibliographystyle{iclr2027_conference}
\bibliography{AALN_references}

@article{ba2016layer,
  title={{Layer normalization}},
  author={Ba, Jimmy Lei and Kiros, Jamie Ryan and Hinton, Geoffrey E},
  journal={arXiv preprint arXiv:1607.06450},
  year={2016}
}

@article{barron1993universal,
  title={{Universal approximation bounds for superpositions of a sigmoidal function}},
  author={Barron, Andrew R},
  journal={IEEE Transactions on Information theory},
  volume={39},
  number={3},
  pages={930--945},
  year={1993},
  publisher={IEEE}
}

@article{baroni2025removal,
  title={{Transformers Don't Need LayerNorm at Inference Time: Scaling LayerNorm Removal to GPT-2 XL and the Implications for Mechanistic Interpretability}},
  author={Baroni, Luca and Khara, Galvin and Schaeffer, Joachim and Subkhankulov, Marat and Heimersheim, Stefan},
  journal={arXiv preprint arXiv:2507.02559},
  year={2025}
}

@inproceedings{brody2023expressivity,
  title={{On the expressivity role of LayerNorm in transformers’ attention}},
  author={Brody, Shaked and Alon, Uri and Yahav, Eran},
  booktitle={Findings of the Association for Computational Linguistics: ACL 2023},
  pages={14211--14221},
  year={2023}
}

@article{cybenko1989approximation,
  title={{Approximation by superpositions of a sigmoidal function}},
  author={Cybenko, George},
  journal={Mathematics of control, signals and systems},
  volume={2},
  number={4},
  pages={303--314},
  year={1989},
  publisher={Springer}
}

@inproceedings{glorot2010difficulty,
  title={{Understanding the difficulty of training deep feedforward neural networks}},
  author={Glorot, Xavier and Bengio, Yoshua},
  booktitle={Proceedings of the thirteenth international conference on artificial intelligence and statistics},
  pages={249--256},
  year={2010},
  organization={JMLR Workshop and Conference Proceedings}
}

@inproceedings{guo2024slab,
  author={Jialong Guo and Xinghao Chen and Yehui Tang and Yunhe Wang},
  title={{SLAB: Efficient Transformers with Simplified Linear Attention and Progressive Re-parameterized Batch Normalization}},
  year={2024},
  booktitle={Proceedings of the 41st International Conference on Machine Learning},
  series={Proceedings of Machine Learning Research},
  volume={235},
  pages={16802--16812},
  url={https://proceedings.mlr.press/v235/guo24a.html}
}

@inproceedings{he2015delving,
  title={{Delving deep into rectifiers: Surpassing human-level performance on imagenet classification}},
  author={He, Kaiming and Zhang, Xiangyu and Ren, Shaoqing and Sun, Jian},
  booktitle={Proceedings of the IEEE international conference on computer vision},
  pages={1026--1034},
  year={2015}
}

@inproceedings{he2016deep,
  title={{Deep residual learning for image recognition}},
  author={He, Kaiming and Zhang, Xiangyu and Ren, Shaoqing and Sun, Jian},
  booktitle={Proceedings of the IEEE conference on computer vision and pattern recognition},
  pages={770--778},
  year={2016}
}

@article{hornik1991approximation,
  title={{Approximation capabilities of multilayer feedforward networks}},
  author={Hornik, Kurt},
  journal={Neural networks},
  volume={4},
  number={2},
  pages={251--257},
  year={1991},
  publisher={Elsevier}
}

@inproceedings{ioffe2015batch,
  title={{Batch normalization: Accelerating deep network training by reducing internal covariate shift}},
  author={Ioffe, Sergey and Szegedy, Christian},
  booktitle={International conference on machine learning},
  pages={448--456},
  year={2015},
  organization={pmlr}
}

@inproceedings{jiang2023pre,
  author={Zixuan Jiang and Jiaqi Gu and Hanqing Zhu and David Z. Pan},
  title={{Pre-RMSNorm and Pre-CRMSNorm Transformers: Equivalent and Efficient Pre-LN Transformers}},
  year={2023},
  booktitle={Advances in Neural Information Processing Systems},
  volume={36},
  pages={45777--45793},
  url={https://papers.nips.cc/paper_files/paper/2023/hash/8f1bacee31caf990a4f08d84f0ccb322-Abstract-Conference.html}
}

@misc{kanavalau2026gated,
  author={Andrei Kanavalau and Carmen Amo Alonso and Sanjay Lall},
  title={{Gated Normalization Removal and Scale Anchoring in Pre-Norm Transformers}},
  year={2026},
  eprint={2602.10408},
  archivePrefix={arXiv},
  note={arXiv version v2},
  url={https://arxiv.org/abs/2602.10408v2}
}

@article{leshno1993nonpolynomial,
  title={{Multilayer feedforward networks with a nonpolynomial activation function can approximate any function}},
  author={Leshno, Moshe and Lin, Vladimir Ya and Pinkus, Allan and Schocken, Shimon},
  journal={Neural networks},
  volume={6},
  number={6},
  pages={861--867},
  year={1993},
  publisher={Elsevier}
}

@article{lin2018resnet,
  title={{Resnet with one-neuron hidden layers is a universal approximator}},
  author={Lin, Hongzhou and Jegelka, Stefanie},
  journal={Advances in neural information processing systems},
  volume={31},
  year={2018}
}

@article{lu2017expressive,
  title={{The expressive power of neural networks: A view from the width}},
  author={Lu, Zhou and Pu, Hongming and Wang, Feicheng and Hu, Zhiqiang and Wang, Liwei},
  journal={Advances in neural information processing systems},
  volume={30},
  year={2017}
}

@inproceedings{ni2024nonlinearity,
  author={Yunhao Ni and Yuxin Guo and Junlong Jia and Lei Huang},
  title={{On the Nonlinearity of Layer Normalization}},
  year={2024},
  booktitle={Proceedings of the 41st International Conference on Machine Learning},
  series={Proceedings of Machine Learning Research},
  volume={235},
  pages={37957--37998},
  url={https://proceedings.mlr.press/v235/ni24b.html}
}

@article{ni2025pln,
  title={{Parallel Layer Normalization for Universal Approximation}},
  author={Ni, Yunhao and Guo, Yuxin and Liu, Yuhe and Sun, Wenxin and Luo, Jie and Wu, Wenjun and Huang, Lei},
  journal={arXiv preprint arXiv:2505.13142},
  year={2025},
  eprint={2505.13142},
  archivePrefix={arXiv},
  url={https://arxiv.org/abs/2505.13142v2}
}

@inproceedings{park2021minimum,
  author={Sejun Park and Chulhee Yun and Jaeho Lee and Jinwoo Shin},
  title={{Minimum Width for Universal Approximation}},
  year={2021},
  booktitle={International Conference on Learning Representations},
  url={https://openreview.net/forum?id=O-XJwyoIF-k}
}

@article{riechers2024geometry,
  title={{Geometry and dynamics of LayerNorm}},
  author={Riechers, Paul M},
  journal={arXiv preprint arXiv:2405.04134},
  year={2024}
}

@inproceedings{stollenwerk2026relationship,
  title={{On the mathematical relationship between layer normalization and dynamic activation functions}},
  author={Stollenwerk, Felix},
  booktitle={Proceedings of the 19th Conference of the European Chapter of the Association for Computational Linguistics (Volume 2: Short Papers)},
  pages={674--681},
  year={2026}
}

@article{xu2019understanding,
  title={{Understanding and improving layer normalization}},
  author={Xu, Jingjing and Sun, Xu and Zhang, Zhiyuan and Zhao, Guangxiang and Lin, Junyang},
  journal={Advances in neural information processing systems},
  volume={32},
  year={2019}
}

@article{zhang2019rmsnorm,
  title={{Root mean square layer normalization}},
  author={Zhang, Biao and Sennrich, Rico},
  journal={Advances in neural information processing systems},
  volume={32},
  year={2019}
}

@inproceedings{zhu2025transformers,
  title={{Transformers without normalization}},
  author={Zhu, Jiachen and Chen, Xinlei and He, Kaiming and LeCun, Yann and Liu, Zhuang},
  booktitle={2025 IEEE/CVF Conference on Computer Vision and Pattern Recognition (CVPR)},
  pages={14901--14911},
  year={2025},
  organization={IEEE}
}
\newpage

\appendix

\section{Proofs for Section 3}

\subsection{Proof of Theorem~\ref{thm:lifted-rmsn-linearization}}
\label{app:lifted-rmsn-proof}

\begin{proof}[Proof of Theorem~\ref{thm:lifted-rmsn-linearization}]
Put $d=N+1$ and define
\begin{equation}
 \bm{W}_{1,M}=\begin{bmatrix}\bm{I}_N\\\zero^\top\end{bmatrix},
 \qquad \bm{b}_{1,M}=\begin{bmatrix}\zero_N\\M\end{bmatrix},
 \qquad
 \bm{W}_{2,M}=\frac{|M|}{\sqrt d}
 \begin{bmatrix}\bm{W}^\star&\zero\end{bmatrix},
 \qquad \bm{b}_{2,M}=\zero.
 \label{eq:rms-one-coordinate-maps}
\end{equation}
Then
\begin{equation}
 f_M(\bm{h})
 =\left(1+\frac{\norm{\bm{h}}^2}{M^2}\right)^{-1/2}
 \bm{W}^\star\bm{h}.
 \label{eq:rms-one-coordinate-identity}
\end{equation}
Since $\mathcal H$ is compact, the elementary inequality
 $0\le 1-(1+u)^{-1/2}\le u/2$ for $u\ge0$ gives
\Eqn\ref{eq:lifted-rmsn-linearization-rate}.
\end{proof}

\section{Proofs for Section 4}
\label{app:rms-an-proofs}

\subsection{Proof of Theorem~\ref{thm:rms-an-unbounded}}

\begin{proof}[Proof of Theorem~\ref{thm:rms-an-unbounded}]
Write $f(\bm{x})=\bm{W}_2\bm{z}(\bm{x})+\bm{b}_2$, where
$\bm{z}(\bm{x})=\phi(\bm{W}_1\bm{x}+\bm{b}_1)$.  For an attainable value $M$,
append one constant activated coordinate and set $d=N+1$:
\begin{equation}
 \bm{y}_M(\bm{x})=
 \begin{bmatrix}\bm{z}(\bm{x})\\M\end{bmatrix},
 \qquad
 \widehat f_M(\bm{x})=
 \frac{|M|}{\sqrt d}
 \begin{bmatrix}\bm{W}_2&\zero\end{bmatrix}
 \operatorname{RMSN}(\bm{y}_M(\bm{x}))+\bm{b}_2.
\end{equation}
Then
\begin{equation}
 \widehat f_M(\bm{x})
 =\left(1+\frac{\norm{\bm{z}(\bm{x})}^2}{M^2}\right)^{-1/2}
 \bm{W}_2\bm{z}(\bm{x})+\bm{b}_2.
 \label{eq:rms-an-unbounded-factor}
\end{equation}
Since $\bm{z}(K)$ is bounded, the inequality used in
Appendix~\ref{app:lifted-rmsn-proof} gives \Eqn\ref{eq:rms-an-unbounded-rate}.
\end{proof}

\subsection{Proof of Theorem~\ref{thm:rms-an-general}}

\begin{proof}[Proof of Theorem~\ref{thm:rms-an-general}]
Choose $a\in\phi(\R)\setminus\{0\}$ and append $r$ copies:
\begin{equation}
 \bm{y}_r(\bm{x})=
 \begin{bmatrix}\bm{z}(\bm{x})\\a\one_r\end{bmatrix}
 \in\R^{N+r}.
\end{equation}
With $d=N+r$, define
\begin{equation}
 \widehat f_r(\bm{x})=
 \frac{\sqrt r|a|}{\sqrt d}
 \begin{bmatrix}\bm{W}_2&\zero\end{bmatrix}
 \operatorname{RMSN}(\bm{y}_r(\bm{x}))+\bm{b}_2.
\end{equation}
The signal factor is
\begin{equation}
 \left(1+\frac{\norm{\bm{z}(\bm{x})}^2}{ra^2}\right)^{-1/2}
 =1+O_f(r^{-1})
 \label{eq:rms-an-replicated-factor}
\end{equation}
uniformly on $K$, proving the claim.
\end{proof}

\section{Proofs for Section 5}
\label{app:rms-pn-proofs}

\subsection{Proof of Theorem~\ref{thm:rms-affine-pre}}

\begin{proof}[Proof of Theorem~\ref{thm:rms-affine-pre}]
Write $f(\bm{x})=\bm{W}_2\phi(\bm{t}(\bm{x}))+\bm{b}_2$ and put $d=N+1$.
Apply RMSNorm to $[\bm{t}(\bm{x})^\top,M]^\top$ and choose the affine scale
\begin{equation}
 \bm{\alpha}_M=\frac{|M|}{\sqrt d}
 \begin{bmatrix}\one_N\\0\end{bmatrix}.
\end{equation}
The first $N$ coordinates after affine scaling are
$\lambda_M(\bm{t}(\bm{x}))\bm{t}(\bm{x})$, where
$\lambda_M(\bm{t})=(1+\norm{\bm{t}}^2/M^2)^{-1/2}$, and the added coordinate
has zero readout weight.  Uniform continuity of $\phi$ proves convergence;
local Lipschitz continuity and $\lambda_M-1=O(M^{-2})$ give the stated rate.
\end{proof}

\subsection{Proof of Proposition~\ref{prop:fixed-width-obstruction}}
\label{app:rms-fixed-width-obstruction}

\begin{proof}[Proof of Proposition~\ref{prop:fixed-width-obstruction}]
Every RMSNorm output $\bm{u}\in\R^d$ satisfies
$\norm{\bm{u}}^2=d$ when nonzero, hence $|u_i|\le\sqrt d$; the same coordinate
bound holds at zero.  Thus $\psi_d(u_i)=0$ for every coordinate, and every
no-affine RMS-PN network has only its final bias left.  By contrast,
$\psi_d(\bm{w}^\top\bm{x}+c)=\operatorname{ReLU}
(\bm{w}^\top\bm{x}+c-\sqrt d)$ is nonconstant for suitable parameters in a
one-hidden-layer MLP.
\end{proof}

\subsection{Proof of Theorem~\ref{thm:rms-homogeneous-pre}}

\begin{proof}[Proof of Theorem~\ref{thm:rms-homogeneous-pre}]
With the notation above, the first $N$ no-affine RMSNorm coordinates are
$c_M\lambda_M(\bm{t}(\bm{x}))\bm{t}(\bm{x})$, where
$c_M=\sqrt{N+1}/|M|$.  Positive homogeneity moves $c_M$ through $\phi$, and
the readout rescales by $c_M^{-1}$.  The remaining factor
$\lambda_M=1+O(M^{-2})$ proves the result exactly as above.
\end{proof}

\subsection{Proof of Theorem~\ref{thm:rms-plain-pre-general}}

\begin{proof}[Proof of Theorem~\ref{thm:rms-plain-pre-general}]
Append $r$ unit coordinates to the preactivation:
\begin{equation}
 \bm{y}_r(\bm{x})=
 \begin{bmatrix}\bm{t}(\bm{x})\\\one_r\end{bmatrix}.
\end{equation}
The first $N$ RMSNorm coordinates are
\begin{equation}
 \bm{u}_r(\bm{x})=
 \left(\frac{N+r}{\norm{\bm{t}(\bm{x})}^2+r}\right)^{1/2}
 \bm{t}(\bm{x})
 =\bm{t}(\bm{x})+O_f(r^{-1})
 \label{eq:rms-pn-replicated-factor}
\end{equation}
uniformly on $K$.  Define
$\widehat f_r=[\bm{W}_2\ \zero]\phi(\operatorname{RMSN}(\bm{y}_r))+\bm{b}_2$.
Uniform continuity of $\phi$ proves convergence, and local Lipschitz continuity
gives \Eqn\ref{eq:rms-plain-pre-rate} for all sufficiently large $r$.
\end{proof}

\section{Proofs for Section 6}
\label{app:ln-counterparts-proofs}

\subsection{LayerNorm counterpart of Theorem~\ref{thm:lifted-rmsn-linearization}}
\label{app:lifted-ln-proof}

\begin{proof}[Derivation of the two-coordinate LN linearization]
Put $d=N+2$, let $s(\bm{h})=\one_N^\top\bm{h}$, and define
\begin{equation}
 \bm{W}_{1,M}=\begin{bmatrix}
  \bm{I}_N\\-\frac12\one_N^\top\\-\frac12\one_N^\top
 \end{bmatrix},
 \qquad
 \bm{b}_{1,M}=\begin{bmatrix}\zero_N\\M\\-M\end{bmatrix},
 \qquad
 \bm{W}_{2,M}=\frac{\sqrt2M}{\sqrt d}
 \begin{bmatrix}\bm{W}^\star&\zero&\zero\end{bmatrix},
 \qquad \bm{b}_{2,M}=\zero.
 \label{eq:lifted-ln-maps}
\end{equation}
The augmented vector is already centered:
\begin{equation}
    \one_d^\top(\bm{W}_{1,M}\bm{h}+\bm{b}_{1,M})=0.
\end{equation}
Moreover, the signal-dependent part is orthogonal to the direction defined by
the two augmentation coordinates, and hence
\begin{equation}
    \norm{\bm{W}_{1,M}\bm{h}+\bm{b}_{1,M}}^2
    =2M^2+\norm{\bm{h}}^2+\frac12s(\bm{h})^2.
    \label{eq:lifted-ln-radius}
\end{equation}
Using \Eqn\ref{eq:exact-ln}, we therefore obtain the exact identity
\begin{equation}
 f_M(\bm{h})
 =
 \left(
  1+\frac{\norm{\bm{h}}^2+\frac12s(\bm{h})^2}{2M^2}
 \right)^{-1/2}\bm{W}^\star\bm{h}.
 \label{eq:lifted-ln-exact-identity}
\end{equation}
Since $\mathcal H$ is compact, both $\norm{\bm{h}}$ and
$\norm{\bm{h}}^2+\frac12s(\bm{h})^2$ are uniformly bounded on $\mathcal H$.  The elementary
inequality
\begin{equation}
    0\le 1-(1+u)^{-1/2}\le\frac{u}{2},
    \qquad u\ge0,
\end{equation}
then gives the uniform $O(M^{-2})$ error asserted in Section~6.1.
\end{proof}

\subsection{Proof of Theorem~\ref{thm:ln-an-unbounded}}

\begin{proof}[AN affine equivalence for LayerNorm]
For an affine AN network,
\begin{equation}
 \bm{W}_2\operatorname{LN}_{\bm{\alpha},\bm{\beta}}(\bm{y})+\bm{b}_2
 =\bm{W}_2\operatorname{diag}(\bm{\alpha})\operatorname{LN}(\bm{y})
  +(\bm{W}_2\bm{\beta}+\bm{b}_2).
\end{equation}
Absorbing $\bm{\alpha}$ and $\bm{\beta}$ into the linear readout and output bias gives
$\mathcal P_N^{\mathrm{AN,aff}}(\phi)
\subseteq\mathcal P_N^{\mathrm{AN,0}}(\phi)$.  Conversely, every no-affine
AN network is obtained from the affine class by choosing
$\bm{\alpha}=\one_N$ and $\bm{\beta}=\zero$.  This proves the reverse inclusion and hence
the equality.
\end{proof}

\label{app:feature-centered-readout-proof}

\begin{proof}[Proof of Lemma~\ref{lem:post-zero-sum}]
Write $f(\bm{x})=\bm{W}_2\bm{z}(\bm{x})+\bm{b}_2$, where
$\bm{z}(\bm{x})=\phi(\bm{W}_1\bm{x}+\bm{b}_1)\in\R^N$.
Choose any $c\in\R$ and append the constant
activated feature $\phi(c)$:
\begin{equation}
 \bar{\bm{W}}_1=
 \begin{bmatrix}\bm{W}_1\\\zero^\top\end{bmatrix},
 \qquad
 \bar{\bm{b}}_1=
 \begin{bmatrix}\bm{b}_1\\c\end{bmatrix},
 \qquad
 \bar{\bm{z}}(\bm{x})=
 \begin{bmatrix}\bm{z}(\bm{x})\\\phi(c)\end{bmatrix}.
\end{equation}
Set
\begin{equation}
 \bar{\bm{W}}_2=
 \begin{bmatrix}\bm{W}_2&-\bm{W}_2\one_N\end{bmatrix},
 \qquad
 \bar{\bm{b}}_2=\bm{b}_2+(\bm{W}_2\one_N)\phi(c).
\end{equation}
Then
\begin{equation}
 \bar{\bm{W}}_2\bar{\bm{z}}(\bm{x})+\bar{\bm{b}}_2
 =\bm{W}_2\bm{z}(\bm{x})-(\bm{W}_2\one_N)\phi(c)
 +\bm{b}_2+(\bm{W}_2\one_N)\phi(c)
 =f(\bm{x}),
\end{equation}
while
\begin{equation}
 \bar{\bm{W}}_2\one_{N+1}
 =\bm{W}_2\one_N-\bm{W}_2\one_N=\zero.
\end{equation}
Thus both the representation and the feature-centering condition hold
exactly.
\end{proof}

\label{app:post-unbounded-proof}

\begin{proof}[Proof of Lemma~\ref{lem:ln-an-zero-sum-conditional}]
Compactness of $K$ and continuity of $\phi$ imply that $\bm{z}(K)$ is bounded.
For an attainable $M\in\phi(\R)$, choose $c_M\in\R$ with
$\phi(c_M)=M$ and append one constant activated coordinate:
\begin{equation}
 \bm{y}_M(\bm{x})=
 \begin{bmatrix}\bm{z}(\bm{x})\\M\end{bmatrix}
 \in\R^d,
 \qquad d=n+1.
 \label{eq:post-augmented-vector}
\end{equation}
This is a valid AN feature vector: append a zero row to $\bm{B}$ and the bias
$c_M$ to $\bm{b}$.  Let
$s(\bm{x})=\one_n^\top\bm{z}(\bm{x})$,
$v(\bm{x})=\norm{\bm{z}(\bm{x})}^2$, and
$\sigma_M(\bm{x})=\sigma(\bm{y}_M(\bm{x}))$.  With
\begin{equation}
 \tau_M=\frac{|M|\sqrt n}{d},
 \label{eq:post-augmentation-target-scale}
\end{equation}
a direct calculation gives
\begin{align}
 \sigma_M(\bm{x})^2
 &=\frac{v(\bm{x})+M^2}{d}-\frac{(s(\bm{x})+M)^2}{d^2},\label{eq:post-augmentation-variance}\\
 \frac{\sigma_M(\bm{x})^2}{\tau_M^2}
 &=1-\frac{2s(\bm{x})}{nM}
   +\frac{dv(\bm{x})-s(\bm{x})^2}{nM^2}.
 \label{eq:post-augmentation-relative-variance}
\end{align}
The functions $s$ and $v$ are uniformly bounded on $K$.  Hence there are
constants $C_0,R_0>0$ such that every attainable $|M|\ge R_0$ satisfies
\begin{equation}
 \sup_{\bm{x}\in K}
 \left|\frac{\sigma_M(\bm{x})^2}{\tau_M^2}-1\right|
 \le \frac{C_0}{|M|}.
 \label{eq:post-augmentation-relative-variance-bound}
\end{equation}
Increasing $R_0$ if necessary makes the ratio lie in $[1/2,3/2]$.
Since $u\mapsto u^{-1/2}$ is Lipschitz on this interval, there is $C_1>0$
such that
\begin{equation}
 \sup_{\bm{x}\in K}
 \left|\frac{\tau_M}{\sigma_M(\bm{x})}-1\right|
 \le \frac{C_1}{|M|}.
 \label{eq:post-augmentation-relative-scale}
\end{equation}

Define the width-$d$ AN network
\begin{equation}
 \widehat g_M(\bm{x})
 =\tau_M
  \begin{bmatrix}\bm{A}&\zero\end{bmatrix}
  \operatorname{LN}(\bm{y}_M(\bm{x}))+\bm{a}.
 \label{eq:post-augmented-network}
\end{equation}
Because $[\bm{A}\;\zero]\one_d=\bm{A}\one_n=\zero$, the readout cancels the
mean subtracted by LN; the augmentation coordinate itself has zero direct readout weight.
Therefore
\begin{equation}
 \widehat g_M(\bm{x})
 =\frac{\tau_M}{\sigma_M(\bm{x})}
 \bm{A}\bm{z}(\bm{x})+\bm{a}.
 \label{eq:post-augmentation-decoding}
\end{equation}
Combining \Eqn\ref{eq:post-augmentation-relative-scale} with boundedness of
$\bm{A}\bm{z}(K)$ yields
\begin{equation}
 \sup_{\bm{x}\in K}\norm{\widehat g_M(\bm{x})-g(\bm{x})}
 \le
 \sup_{\bm{x}\in K}\left|\frac{\tau_M}{\sigma_M(\bm{x})}-1\right|
 \sup_{\bm{x}\in K}\norm{\bm{A}\bm{z}(\bm{x})}
 \le \frac{C_g}{|M|}.
\end{equation}
The constructed width is $d=n+1$.
\end{proof}

\begin{proof}[Proof of Theorem~\ref{thm:ln-an-unbounded}]
Lemma~\ref{lem:post-zero-sum} represents the target $f$ exactly as a width-
$(N+1)$ MLP with a zero-sum readout.  Apply
Lemma~\ref{lem:ln-an-zero-sum-conditional} with $n=N+1$.  The resulting LN-AN
network has width $n+1=N+2$ and satisfies
\Eqn\ref{eq:ln-an-unbounded-rate}.  Since the unbounded range of $\phi$ contains
attainable values of arbitrarily large magnitude, the error can be made smaller
than any prescribed $\varepsilon>0$.
\end{proof}

\subsection{LayerNorm counterpart of Theorem~\ref{thm:rms-an-general}}
\label{app:post-two-value-proof}

\begin{proof}[General AN case of Proposition~\ref{prop:ln-counterparts}]
Write the target as
\begin{equation}
 f(\bm{x})=\bm{W}_2\bm{z}(\bm{x})+\bm{b}_2,
 \qquad
 \bm{z}(\bm{x})
 =\phi(\bm{W}_1\bm{x}+\bm{b}_1)\in\R^N.
\end{equation}
Choose $\xi_a,\xi_b\in\R$ such that
$a=\phi(\xi_a)\ne b=\phi(\xi_b)$.
For $r\ge1$, append $r$ copies of each value:
\begin{equation}
 \bm{y}_r(\bm{x})=
 \begin{bmatrix}
  \bm{z}(\bm{x})\\a\one_r\\b\one_r
 \end{bmatrix}
 \in\R^D,
 \qquad
 D=N+2r.
 \label{eq:app-two-value-vector}
\end{equation}
These constant activated coordinates are realized by constant preactivations
$\xi_a$ and $\xi_b$.  Give all of them zero readout weight and put
\begin{equation}
 \widetilde{\bm{W}}_{2,r}
 =\begin{bmatrix}\bm{W}_2&\zero\end{bmatrix}.
\end{equation}
Then
\begin{equation}
 \widetilde{\bm{W}}_{2,r}\bm{y}_r(\bm{x})
 =\bm{W}_2\bm{z}(\bm{x}),
 \qquad
 \widetilde{\bm{W}}_{2,r}\one_D=\bm{W}_2\one_N.
 \label{eq:app-two-value-readout}
\end{equation}

Write
\begin{equation}
 s(\bm{x})=\one_N^\top\bm{z}(\bm{x}),
 \quad
 v(\bm{x})=\norm{\bm{z}(\bm{x})}^2,
 \quad
 A_r=r(a^2+b^2),
 \quad
 B_r=r(a+b),
\end{equation}
and define the augmentation-only baseline mean and variance by
\begin{equation}
 m_r=\frac{B_r}{D},
 \qquad
 \tau_r^2=\frac{A_r}{D}-\frac{B_r^2}{D^2}.
 \label{eq:app-two-value-baseline}
\end{equation}
The full mean is $\mu_r(\bm{x})=m_r+s(\bm{x})/D$, and the full variance is
\begin{align}
 \sigma_r(\bm{x})^2
 &=\frac{v(\bm{x})+A_r}{D}
   -\frac{(s(\bm{x})+B_r)^2}{D^2}\\
 &=\tau_r^2+\Delta_r(\bm{x}),
 \label{eq:app-two-value-variance}
\end{align}
where
\begin{equation}
 \Delta_r(\bm{x})
 =\frac{v(\bm{x})}{D}
  -\frac{2B_rs(\bm{x})+s(\bm{x})^2}{D^2}.
 \label{eq:app-two-value-perturbation}
\end{equation}
Because $\bm{z}(K)$ is bounded,
$\sup_{\bm{x}\in K}|\Delta_r(\bm{x})|=O(r^{-1})$, while
\begin{equation}
 \tau_r^2\longrightarrow\frac{(a-b)^2}{4}>0.
\end{equation}
Hence, uniformly on $K$,
\begin{equation}
 \frac{\tau_r}{\sigma_r(\bm{x})}=1+O(r^{-1}).
 \label{eq:app-two-value-ratio}
\end{equation}
Finally, define
\begin{equation}
 \widehat f_r(\bm{x})
 =\tau_r\widetilde{\bm{W}}_{2,r}
  \operatorname{LN}(\bm{y}_r(\bm{x}))+\bm{b}_2
  +m_r\bm{W}_2\one_N.
\end{equation}
Put $\rho_r(\bm{x})=\tau_r/\sigma_r(\bm{x})$.  Using
\Eqn\ref{eq:app-two-value-readout} and
$\mu_r=m_r+s/D$, the approximation error is
\begin{equation}
 \widehat f_r(\bm{x})-f(\bm{x})
 =\bigl(\rho_r(\bm{x})-1\bigr)\bm{W}_2\bm{z}(\bm{x})
 +\left[
   \bigl(1-\rho_r(\bm{x})\bigr)m_r
   -\rho_r(\bm{x})\frac{s(\bm{x})}{D}
  \right]\bm{W}_2\one_N.
\end{equation}
Now $m_r=O(1)$, $D^{-1}=O(r^{-1})$, and $s$ is uniformly bounded on $K$.
Together with \Eqn\ref{eq:app-two-value-ratio}, these facts give
$\supnorm{\widehat f_r-f}=O_f(r^{-1})$, completing the proof with width
$D=N+2r$.
\end{proof}

\subsection{LayerNorm counterpart of Theorem~\ref{thm:rms-affine-pre}}
\label{app:affine-pre-proof}

\begin{proof}[Affine-PN case of Proposition~\ref{prop:ln-counterparts}]
Write $f(\bm{x})=\bm{W}_2\phi(\bm{t}(\bm{x}))+\bm{b}_2$, where
$\bm{t}(\bm{x})=\bm{W}_1\bm{x}+\bm{b}_1\in\R^N$, and apply the first
layer parameters $\bm{W}_{1,M},\bm{b}_{1,M}$ from
\Eqn\ref{eq:lifted-ln-maps} to $\bm{t}(\bm{x})$.  Put
$d=N+2$, $c_M=\sqrt2M/\sqrt d$, and choose affine LN
parameters
\begin{equation}
    \bm{\alpha}_M=c_M
    \begin{bmatrix}\one_N\\0\\0\end{bmatrix},
    \qquad \bm{\beta}_M=\zero.
\end{equation}
By \Eqn\ref{eq:lifted-ln-exact-identity} with
$\bm{W}^\star=\bm{I}_N$, the first $N$
coordinates after affine LN are
$\lambda_M(\bm{t}(\bm{x}))\bm{t}(\bm{x})$, where
$\lambda_M(\bm{t})\to1$ uniformly on $\bm{t}(K)$ with error $O(M^{-2})$.  The last
two activated coordinates are assigned zero readout weight, so
\begin{equation}
 \widehat f_M(\bm{x})=
 \begin{bmatrix}\bm{W}_2&\zero&\zero\end{bmatrix}
 \phi\!\left(\operatorname{LN}_{\bm{\alpha}_M,\zero}
 (\bm{W}_{1,M}\bm{t}(\bm{x})+\bm{b}_{1,M})\right)+\bm{b}_2
 \longrightarrow f(\bm{x})
\end{equation}
uniformly by continuity of $\phi$.  Local Lipschitz continuity gives the
stated rate.
\end{proof}

\subsection{LayerNorm counterpart of Proposition~\ref{prop:fixed-width-obstruction}}
\label{app:ln-fixed-width-obstruction-proof}

\begin{proof}[Finite-width obstruction case of
Proposition~\ref{prop:ln-counterparts}]
Every nondegenerate LN output $\bm{u}\in\R^d$ satisfies
$\norm{\bm{u}}^2=d$, so $u_i\le\sqrt d$ for each coordinate.  Thus
$\psi_d(u_i)=0$ for all $i$.  At a degenerate input, our definition
$\operatorname{LN}(\bm{u})=\zero$ gives the same conclusion.  Hence the network
output is its final bias.

On the other hand,
\begin{equation}
 \psi_d(\bm{w}^\top\bm{x}+c)
 =
 \operatorname{ReLU}(\bm{w}^\top\bm{x}+c-\sqrt d),
\end{equation}
so the threshold is absorbed into the first-layer bias of an ordinary
$\psi_d$-network, which can therefore represent nonconstant functions.
\end{proof}

\subsection{LayerNorm counterpart of Theorem~\ref{thm:rms-homogeneous-pre}}
\label{app:homogeneous-pre-proof}

\begin{proof}[Homogeneous PN case of Proposition~\ref{prop:ln-counterparts}]
Write $f(\bm{x})=\bm{W}_2\phi(\bm{t}(\bm{x}))+\bm{b}_2$, where
$\bm{t}(\bm{x})=\bm{W}_1\bm{x}+\bm{b}_1$, and use the same augmented
preactivation $\bm{W}_{1,M}\bm{t}(\bm{x})+\bm{b}_{1,M}$ as in
\Eqn\ref{eq:lifted-ln-maps}.  Put
$d=N+2$ and $c_M=\sqrt2M/\sqrt d$.  If $\lambda_M(\bm{t})$ denotes the scalar
in \Eqn\ref{eq:lifted-ln-exact-identity} for
$\bm{W}^\star=\bm{I}_N$, then the first $N$
coordinates of $\operatorname{LN}(\bm{W}_{1,M}\bm{t}+\bm{b}_{1,M})$ equal
$c_M^{-1}\lambda_M(\bm{t})\bm{t}$.  Positive homogeneity therefore gives
\begin{equation}
 c_M\begin{bmatrix}\bm{W}_2&\zero&\zero\end{bmatrix}
 \phi\!\left(\operatorname{LN}
 (\bm{W}_{1,M}\bm{t}(\bm{x})+\bm{b}_{1,M})\right)+\bm{b}_2
 =\lambda_M(\bm{t}(\bm{x}))\bm{W}_2\phi(\bm{t}(\bm{x}))+\bm{b}_2.
\end{equation}
The factor $\lambda_M$ converges uniformly to one at rate $O(M^{-2})$ on
the compact set $\bm{t}(K)$, proving the claim.
\end{proof}

\subsection{LayerNorm counterpart of Theorem~\ref{thm:rms-plain-pre-general}}
\label{app:plain-pre-identity-proof}

\begin{lemma}[Replicated-coordinate identity limit]
\label{lem:plain-pre-identity}
Let $\bm{t}(\bm{x})=\bm{W}_1\bm{x}+\bm{b}_1\in\R^N$.
For $r\ge1$, define
\begin{equation}
 \bm{y}_r(\bm{x})=
 \begin{bmatrix}
  \bm{t}(\bm{x})\\-\one_r\\\one_r
 \end{bmatrix}
 \in\R^{N+2r}.
 \label{eq:plain-pre-vector}
\end{equation}
Let $\bm{u}_r(\bm{x})$ be the first $N$ coordinates of
$\operatorname{LN}(\bm{y}_r(\bm{x}))$.  Then
\begin{equation}
    \sup_{\bm{x}\in K}\norm{\bm{u}_r(\bm{x})-\bm{t}(\bm{x})}=O(r^{-1}).
    \label{eq:plain-pre-identity}
\end{equation}
\end{lemma}

\begin{proof}
Write $s(\bm{x})=\one_N^\top\bm{t}(\bm{x})$ and $D=N+2r$.
The replicated values $-1$
and $1$ cancel in the sum, so
\begin{equation}
    \mu_r(\bm{x})=\frac{s(\bm{x})}{D}.
\end{equation}
They contribute $2r$ to the squared norm, and hence
\begin{align}
 \rho_r(\bm{x})^2
 &=\frac{\norm{\bm{t}(\bm{x})}^2+2r}{D}-\frac{s(\bm{x})^2}{D^2}\\
 &=1+\frac{\norm{\bm{t}(\bm{x})}^2-N}{D}-\frac{s(\bm{x})^2}{D^2}.
 \label{eq:app-plain-pre-variance}
\end{align}
Compactness of $\bm{t}(K)$ gives
$\sup_{\bm{x}\in K}|\rho_r(\bm{x})^2-1|=O(r^{-1})$.  Thus, for all sufficiently
large $r$, $\rho_r$ is uniformly bounded away from zero and
$\rho_r^{-1}-1=O(r^{-1})$ uniformly on $K$.

The first $N$ coordinates of $\operatorname{LN}(\bm{y}_r(\bm{x}))$ are
\begin{equation}
    \bm{u}_r(\bm{x})
    =\frac{\bm{t}(\bm{x})-D^{-1}s(\bm{x})\one_N}{\rho_r(\bm{x})}.
\end{equation}
Consequently,
\begin{equation}
 \bm{u}_r(\bm{x})-\bm{t}(\bm{x})
 =\bigl(\rho_r(\bm{x})^{-1}-1\bigr)\bm{t}(\bm{x})
 -\frac{s(\bm{x})}{D\rho_r(\bm{x})}\one_N.
\end{equation}
Both terms are uniformly $O(r^{-1})$, which proves
\Eqn\ref{eq:plain-pre-identity}.
\end{proof}

\begin{proof}[General PN case of Proposition~\ref{prop:ln-counterparts}]
Write $f(\bm{x})=\bm{W}_2\phi(\bm{t}(\bm{x}))+\bm{b}_2$ and use
$\bm{y}_r(\bm{x})$ from
Lemma~\ref{lem:plain-pre-identity}.  Define
\begin{equation}
 \widehat f_r(\bm{x})
 =
 \begin{bmatrix}\bm{W}_2&\zero\end{bmatrix}
 \phi\!\left(\operatorname{LN}(\bm{y}_r(\bm{x}))\right)+\bm{b}_2,
 \label{eq:plain-pre-network}
\end{equation}
where the zero block ignores the $2r$ augmentation coordinates after activation.
By Lemma~\ref{lem:plain-pre-identity}, the first $N$ normalized coordinates
converge uniformly to $\bm{t}(\bm{x})$.  They and $\bm{t}(K)$ lie in a common compact set
for all sufficiently large $r$, so uniform continuity of $\phi$ proves
uniform convergence to $f$.  If $\phi$ is Lipschitz on compact sets,
\Eqn\ref{eq:plain-pre-identity} yields constants $C_f>0$ and $r_f$ for which
the $O_f(r^{-1})$ bound holds whenever $r\ge r_f$.
\end{proof}

\section{Proofs for Section 7}

\subsection{Proof of Proposition~\ref{prop:rms-one-coordinate-necessity}}
\label{app:rms-one-coordinate-lower-bound}

\begin{proof}[Proof of Proposition~\ref{prop:rms-one-coordinate-necessity}]
If $\operatorname{rank}\bm{W}_1<N$, choose a unit
$\bm{v}\in\ker\bm{W}_1$.  The inputs $\bm{v}$ and $-\bm{v}$ have the same
normalized representation, so the error is at least $1$.  Otherwise $d=N$ and
$\bm{W}_1$ is invertible.  Let $\bm{x}_0$ be the unique solution of
$\bm{W}_1\bm{x}_0+\bm{b}_1=\zero$ and write
$\bm{x}_0=t\bm{v}$ with $t\ge0$ and $\norm{\bm{v}}=1$ (choose any unit
$\bm{v}$ if $\bm{x}_0=\zero$).  For
$-1\le s<\min\{t,1\}$, the representations of $s\bm{v}$ lie on the same
nonzero ray, so their RMSNorm outputs coincide.  This segment has diameter at
least one, giving error at least $1/2$; an endpoint of zero variance is handled
by taking a limit.
\paragraph{LayerNorm extension.}
Set $\bm{C}=\bm{P}_d\bm{W}_1$ and $\bm{c}=\bm{P}_d\bm{b}_1$.  Any two inputs
whose centered affine images lie on the same nonzero ray have the same LN
representation and hence the same network output.  Since
$\operatorname{rank}\bm{C}\le d-1$, if $\operatorname{rank}\bm{C}<N$ we may
choose a unit $\bm{v}\in\ker\bm{C}$; the inputs $\bm{v}$ and $-\bm{v}$ then
have the same centered affine image $\bm{c}$, so the approximation error is at
least $1$.  Otherwise $\operatorname{rank}\bm{C}=N$, which forces $d=N+1$,
and $\bm{C}$ is a bijection from $\R^N$ onto $\one_d^\perp$.  Because
$\bm{c}\in\one_d^\perp$, there is a unique $\bm{x}_0$ satisfying
$\bm{C}\bm{x}_0+\bm{c}=\zero$.  Write $\bm{x}_0=t\bm{v}$ with $t\ge0$ and
$\norm{\bm{v}}=1$, choosing any unit $\bm{v}$ when $\bm{x}_0=\zero$.  For
$-1\le s<\min\{t,1\}$,
\begin{equation}
 \bm{C}(s\bm{v})+\bm{c}=(s-t)\bm{C}\bm{v}
 \label{eq:ln-lower-bound-ray}
\end{equation}
lies on one nonzero ray.  The corresponding inputs form a segment of diameter
$1+\min\{t,1\}\ge1$, so their common network output has uniform error at least
$1/2$.  If the limiting endpoint has zero variance, taking points approaching
it gives the same bound.  Thus the LN obstruction holds for every $d\le N+1$.
\end{proof}

\subsection{Proof of Corollary~\ref{cor:density-transfer}}
\label{app:density-transfer-proof}

\begin{proof}[Proof of Corollary~\ref{cor:density-transfer}]
Given a target $f_\star\in C(K;\R^{d_y})$ and $\varepsilon>0$, density of the
one-hidden-layer MLPs gives some $f\in\mathcal M_N(\phi)$ with
$\norm{f_\star-f}_{\infty,K}<\varepsilon/2$.  Apply
Theorem~\ref{thm:rms-an-general}, \ref{thm:rms-affine-pre}, or
\ref{thm:rms-plain-pre-general} for the three RMSNorm claims, and then
Proposition~\ref{prop:ln-counterparts} for their LayerNorm counterparts.  In
each case choose the approximation error below $\varepsilon/2$; the triangle
inequality gives $\norm{f_\star-\widehat f}_{\infty,K}<\varepsilon$.
\end{proof}

\section{Experimental details}
\label{app:numerical-details}

The activation-family experiments use $d_x=d_y=1$, $K=[-1,1]$, and target
MLPs with parameters
\begin{equation}
 \bm{W}_1=1.15\bm{G}_1,
 \qquad \bm{b}_1=0.45\bm{G}_2,
 \qquad \bm{W}_2=\frac{0.7\bm{G}_3+0.25\one_N^\top}{\sqrt N},
 \qquad \bm{b}_2=0.15\bm{G}_4,
\end{equation}
where all entries of $\bm{G}_1,\bm{G}_2,\bm{G}_3,\bm{G}_4$ are standard normal
draws.  We hold each sampled target fixed while varying the augmentation
parameter and set the LN and RMSNorm approximant parameters by the explicit
constructions, without optimization.  We evaluate the maximum absolute output
error on an $8193$-point grid in $K$ using float64 arithmetic and exact
normalization.  This grid maximum estimates, but does not certify, the
continuous uniform error.

Figure~\ref{fig:five-construction-rates} fixes $N=16$ and averages the grid
maximum over six independently sampled targets.  The extended sweep in
Figures~\ref{fig:activation-family-sweep} and
\ref{fig:activation-family-sweep-part2} uses three activations in each of the
five regimes, widths $N\in\{4,8,16,32\}$, six targets per width, and
$M,r\in\{2^0,\ldots,2^{10}\}$.  It contains $3960$ evaluations per
normalization, with no nonfinite errors.  The shifted-ReLU experiments below
use a different target setup, described in their own subsection.

Table~\ref{tab:five-construction-slopes} summarizes one representative
activation per regime.  For each target--width pair, we fit a slope over the
final six log--log points, then report the median and interquartile range of
the $24$ fitted slopes.  A slope near $-q$ is consistent with error decaying
like the control parameter raised to $-q$ over the observed tail.
Because the error constant may depend on the target realization and its width
$N$, these pooled summaries describe trends across the tested widths rather
than an error bound with a constant uniform in $N$.

\begin{table}[ht]
\centering
\small
\setlength{\tabcolsep}{3.5pt}
\caption{Predicted and observed tail slopes for one activation in each regime.
Observed values are medians (interquartile ranges) over $24$ target--width
pairs.}
\label{tab:five-construction-slopes}
\begin{tabular}{lccc}
\toprule
Construction & Predicted RMS/LN & RMS observed & LN observed\\
\midrule
AN, unbounded (ReLU) & $-2/-1$ & $-1.998$ ($-1.999,-1.997$) & $-0.953$ ($-0.977,-0.886$)\\
AN, general ($\tanh$) & $-1/-1$ & $-0.956$ ($-0.975,-0.931$) & $-0.968$ ($-0.982,-0.940$)\\
PN, affine ($\tanh$) & $-2/-2$ & $-1.999$ ($-1.999,-1.997$) & $-1.999$ ($-1.999,-1.998$)\\
PN, homogeneous (ReLU) & $-2/-2$ & $-1.997$ ($-1.999,-1.994$) & $-1.998$ ($-1.999,-1.995$)\\
PN, general ($\tanh$) & $-1/-1$ & $-0.956$ ($-0.970,-0.913$) & $-0.966$ ($-0.979,-0.937$)\\
\bottomrule
\end{tabular}
\end{table}

\subsection{Extended activation-family verification}

Figures~\ref{fig:activation-family-sweep} and
\ref{fig:activation-family-sweep-part2} test two further activations alongside
the representative one in each regime.  Their rows cover unbounded AN,
general AN, affine PN, homogeneous PN, and general PN, respectively.
The first row shows the different $M^{-1}$ LN and $M^{-2}$ RMSNorm trends.
The general AN and PN rows approach the $r^{-1}$ reference slope, while the
affine and homogeneous PN rows approach the common $M^{-2}$ slope.  Some
curves deviate at small augmentation values; across the tested activations,
their tails are broadly consistent with the corresponding predicted rates.

\begin{figure}[!t]
    \centering
    \includegraphics[width=\textwidth]{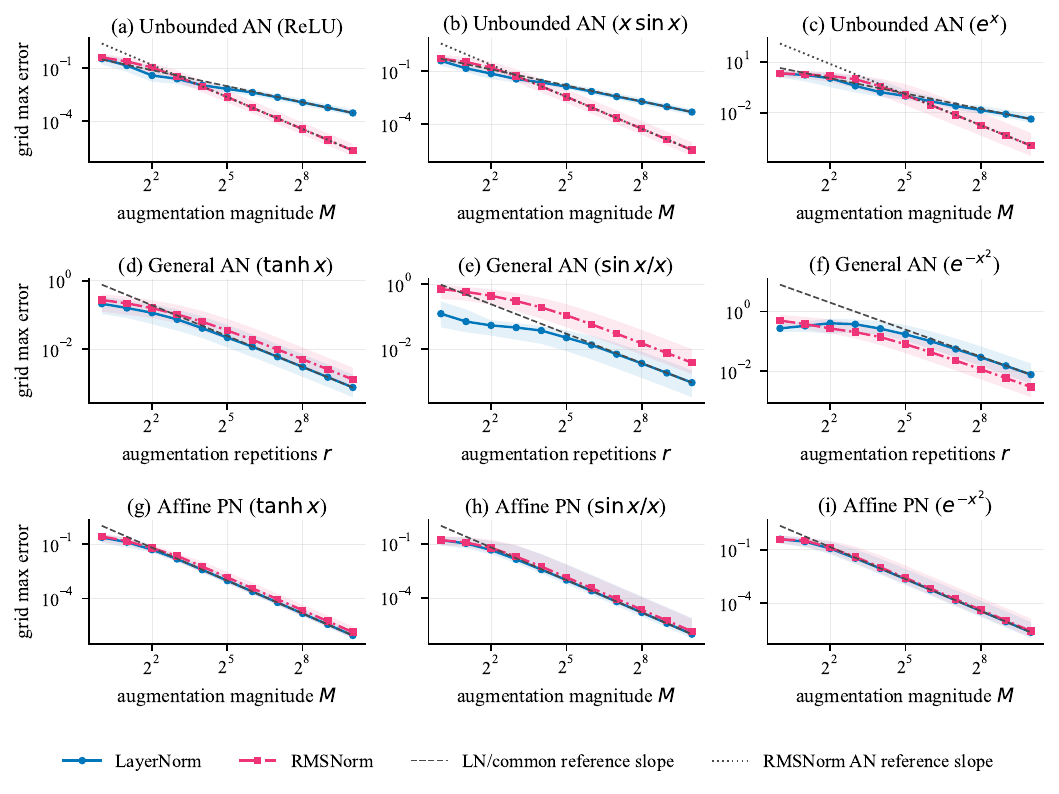}
    \caption{Extended activation-family tests: unbounded AN, general AN, and
    affine PN.  Each blue or pink curve is the median grid maximum error over
    $24$ target--width pairs for LN or RMSNorm, and the shaded region is its
    interquartile range.  Gray dashed and dotted lines illustrate reference
    slopes, not error bounds.  Both axes are logarithmic.}
    \label{fig:activation-family-sweep}
\end{figure}

\begin{figure}[!t]
    \centering
    \includegraphics[width=\textwidth]{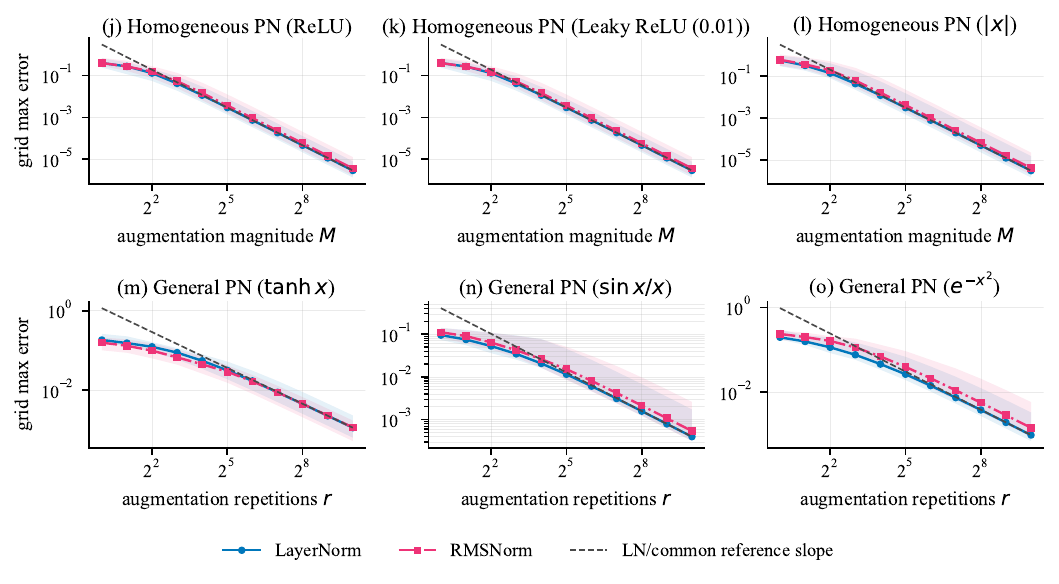}
    \caption{Extended activation-family tests: homogeneous and general PN.
    Curves and shaded regions have the same meaning as in
    Figure~\ref{fig:activation-family-sweep}; the gray dashed lines illustrate
    the common reference slopes.}
    \label{fig:activation-family-sweep-part2}
\end{figure}
\FloatBarrier
\subsection{Shifted-ReLU finite-width barrier}

Both experiments in this subsection fix $K=[-1,1]$ and
$\phi(t)=\operatorname{ReLU}(t-10)$; they use explicit constructions without
fitting parameters.  Figure~\ref{fig:shifted-relu-explicit} takes $N=1$ and
$f(x)=\phi(10+2x)=\operatorname{ReLU}(2x)$.  For normalized width $d\le 10^2$,
the coordinate bound $|u_i|\le\sqrt d$ makes every no-affine PN hidden
activation vanish.  Such networks are constant, so their uniform error for
this target is at least $1$.  This bound is a necessary width barrier, not
the exact width at which the plotted constructions turn on.

Figures~\ref{fig:shifted-relu-explicit}(a) and
\ref{fig:shifted-relu-explicit}(b) show how the constructed RMS-PN and LN-PN
outputs approach the target as $d$ grows.  Figure~\ref{fig:shifted-relu-explicit}(c)
plots their grid maximum errors: after the finite-width obstruction, their
tails are consistent with $d^{-1}$ decay.  With affine parameters, the widths
remain $2$ and $3$, and Figure~\ref{fig:shifted-relu-explicit}(d) shows the
$M^{-2}$ trend.  The growing-width behavior for this fixed activation does
not contradict Proposition~\ref{prop:fixed-width-obstruction}, whose
activation depends on $d$.

\begin{figure}[!htbp]
    \centering
    \includegraphics[width=0.96\textwidth]{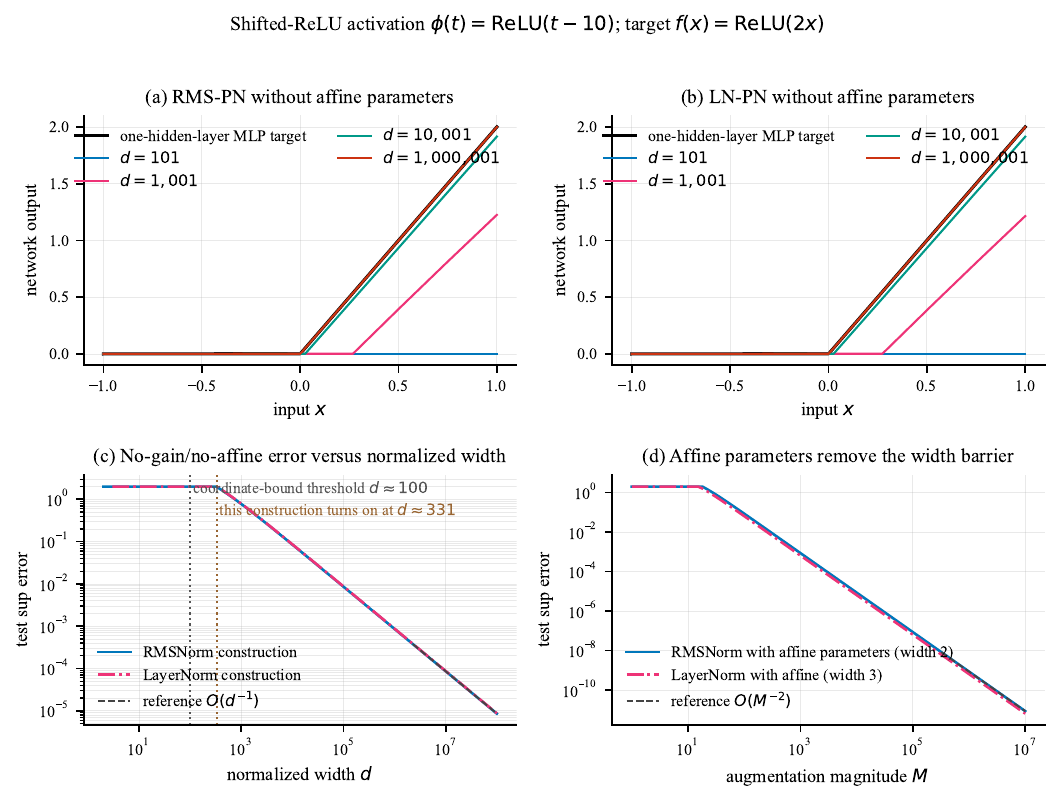}
    \caption{Shifted-ReLU PN constructions for
    $f(x)=\operatorname{ReLU}(2x)$.  (a,b) Function-level recovery as width
    grows.  (c) The coordinate-bound barrier $d=10^2$ and subsequent
    grid-error decay consistent with $d^{-1}$ without affine parameters.
    (d) Fixed-width affine constructions with grid-error decay consistent
    with $M^{-2}$.}
    \label{fig:shifted-relu-explicit}
\end{figure}

\FloatBarrier
Figure~\ref{fig:shifted-relu-random} repeats the comparison with random
targets of widths $N\in\{4,8,16,32\}$, six at each width.  For each target,
let $z(x)=\bm{W}_1x+\bm{b}_1$ with the Gaussian parameters above and set
$f(x)=\bm{W}_2\phi(10+z(x))=\bm{W}_2\operatorname{ReLU}(z(x))$.
For each target, the relative grid maximum error is the maximum absolute
approximation error on the test grid divided by the maximum absolute target
value on that grid.
Figures~\ref{fig:shifted-relu-random}(a) and
\ref{fig:shifted-relu-random}(b) show one target and selected no-affine
approximants.  Across all $24$ targets,
Figure~\ref{fig:shifted-relu-random}(c) shows median relative grid maximum
error decreasing approximately like $d^{-1}$ after the barrier; the
fixed-width affine curves in Figure~\ref{fig:shifted-relu-random}(d) follow
the $M^{-2}$ reference slope.  Shaded bands show interquartile ranges.

\clearpage
\begin{figure}[!ht]
    \centering
    \includegraphics[width=0.96\textwidth]{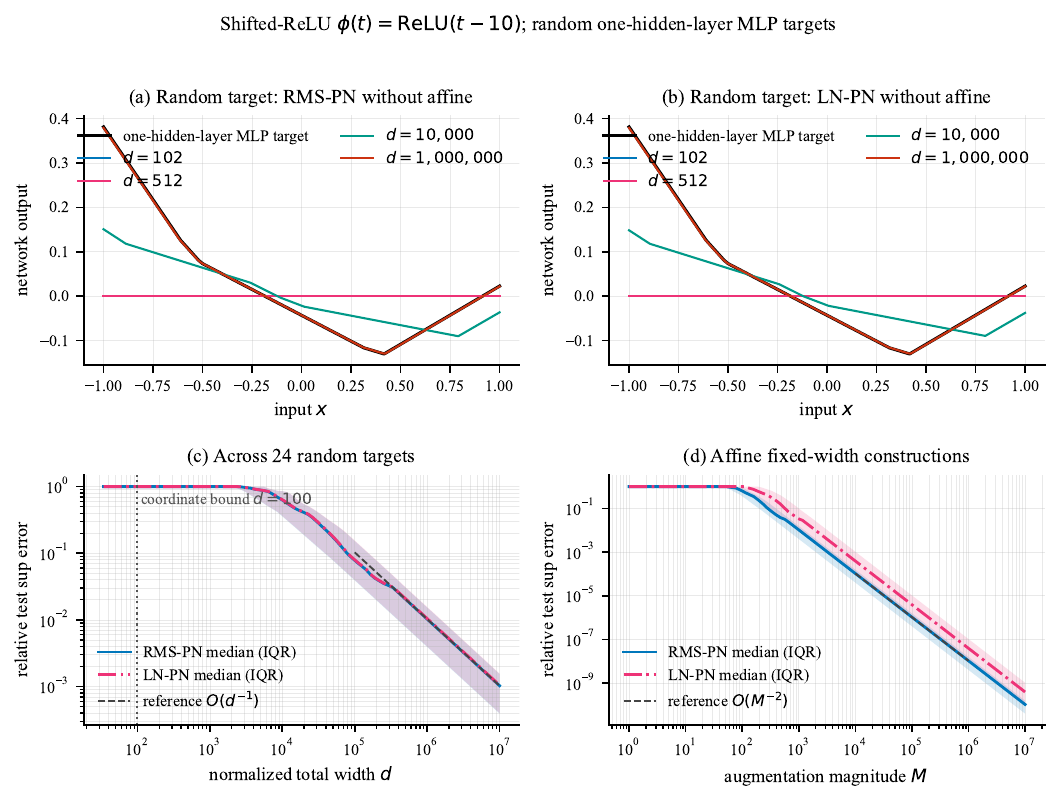}
    \caption{Shifted-ReLU PN verification over $24$ random one-hidden-layer MLP
    targets.  (a,b) One representative target and selected approximants.
    (c,d) Median relative grid maximum errors and interquartile ranges for
    the width-growing no-affine and fixed-width affine constructions,
    respectively.  Gray lines illustrate reference slopes, not error bounds.}
    \label{fig:shifted-relu-random}
\end{figure}

\end{document}